\documentclass[a4paper, 10pt]{article}

\usepackage[english]{babel}
\usepackage[utf8x]{inputenc}
\usepackage[T1]{fontenc}
\usepackage{bbm}
\usepackage{cancel}
\usepackage{amsfonts, amsmath, amssymb, amsthm}
\usepackage{fullpage}
\usepackage[hmarginratio=1:1, left = 20mm, top = 20mm,columnsep=20pt]{geometry}
\usepackage{mathpazo}
\usepackage{algorithm}
\usepackage{algorithmic}
\usepackage{booktabs}
\usepackage{enumitem}
\usepackage{titling}
\usepackage{wrapfig}
\usepackage{mathtools}
\usepackage{mathrsfs}
\usepackage{booktabs}
\usepackage[protrusion=true,expansion=true]{microtype}
\usepackage{subcaption}
\usepackage{xargs}
\usepackage{cite}

\usepackage[colorlinks=true, allcolors=blue]{hyperref}
\theoremstyle{plain}
  \newtheorem{theorem}{\sffamily Theorem}[section]
  \newtheorem{proposition}[theorem]{\sffamily Proposition}
  \newtheorem{lemma}[theorem]{\sffamily Lemma}

  \newtheorem{remark}[theorem]{\sffamily Remark}
  \newtheorem{definition}[theorem]{\sffamily Definition}

\allowdisplaybreaks

\providecommand{\Verr}{\boldsymbol{\mathsf{err}}}
\providecommand{\edge}{\bsf}

\providecommand{\Null}{\operatorname{ker}}
\providecommand{\trace}{\operatorname{trace}}                        
\providecommand{\cov}[1]{\operatorname{Cov{\LRP{#1}} }}
\providecommand{\Sgn}{\operatorname{sgn}}                    
\providecommand{\cov}[1]{ \operatorname{Cov{\LRP{#1}} }}
\providecommand{\Null}{\operatorname{ker}}
\providecommand{\ker}{\operatorname{ker}}
\providecommand{\Range}{\operatorname{range}}                        
\providecommand{\range}{\operatorname{range}}                        
\providecommand{\rank}{\operatorname{rank}}                        
\providecommand{\trace}{\operatorname{trace}}                        
\providecommand{\tr}{\operatorname{tr}}                        

\providecommand{\argmin}{\operatorname*{argmin}}  
\providecommand{\argmax}{\operatorname*{argmax}}  
\providecommand*{\Span}[1]{\operatorname{span}\left\{{#1}\right\}}     
\providecommand{\Dim}{\operatorname{dim}}            
\providecommand{\dim}{\Dim}

\newcommand*{\Op}[1]{\mathsf{#1}} 
\providecommand{\CD}{{\cal D}}
\newcommand*{\LRP}[1]{\left(#1\right)}
\newcommand*{\LRS}[1]{\left[#1\right]}
\newcommand{\xibf}{\boldsymbol{\xi}}
\newcommand{\Amat}{{\textrm{A}}}
\newcommand{\Pmat}{{\textrm{P}}}
\newcommand{\Imat}{{\textrm{I}}}
\newcommand{\Rmat}{{\textrm{R}}}
\newcommand{\Qmat}{{\textrm{Q}}}
\newcommand{\Mmat}{{\textrm{M}}}
\newcommand{\Smat}{{\textrm{S}}}

\newcommand{\bsf}{\mathsf{b}}
\newcommand{\usf}{\mathsf{u}}
\newcommand{\vsf}{\mathsf{v}}
\newcommand{\wsf}{\mathsf{w}}
\newcommand{\psf}{\mathsf{p}}
\newcommand{\xsf}{\mathsf{x}}
\newcommand{\ssf}{\mathsf{s}}
\newcommand{\ysf}{\mathsf{y}}
\newcommand{\zsf}{\mathsf{z}}
\newcommand{\abs}{{\boldsymbol{\mathsf{a}}}}
\newcommand{\ebs}{{\boldsymbol{\mathsf{e}}}}
\newcommand{\bbs}{{\boldsymbol{\mathsf{b}}}}
\newcommand{\ubs}{{\boldsymbol{\mathsf{u}}}}
\newcommand{\vbs}{{\boldsymbol{\mathsf{v}}}}
\newcommand{\wbs}{{\boldsymbol{\mathsf{w}}}}
\newcommand{\pbs}{{\boldsymbol{\mathsf{p}}}}
\newcommand{\xbs}{{\boldsymbol{\mathsf{x}}}}
\newcommand{\ybs}{{\boldsymbol{\mathsf{y}}}}
\providecommand{\CV}{{\cal V}}
\providecommand{\CW}{{\cal W}}
\providecommand{\CI}{{\cal I}}
\newcommand{\zbs}{{\boldsymbol{\mathsf{z}}}}
\providecommand{\CS}{{\cal S}}
\providecommand{\CN}{{\cal N}}
\providecommand{\FP}{\mathfrak{P}}
\providecommand{\CT}{{\cal T}}
\newcommand{\Csv}{\boldsymbol{\CS}}
\providecommand{\bbR}{\mathbb{R}}
\providecommand{\FM}{\mathfrak{M}}
\providecommand{\bbP}{\mathbb{P}}
\newcommand{\Veta}{\boldsymbol{\eta}} %
\providecommand{\bbN}{\mathbb{N}}
\providecommand{\bbW}{\mathbb{W}}
\providecommand{\bbE}{\mathbb{E}}
\newcommand*{\EUSP}[2]{\left<{#1},{#2}\right>} 
\providecommand*{\N}[1]{\left\|{#1}\right\|} 
\newcommand*{\SN}[1]{\left|{#1}\right|}      
\providecommand{\Id}{\Op{Id}}                     

\newcommand{\ra}[1]{\renewcommand{\arraystretch}{#1}}
\usepackage[affil-it]{authblk}

\date{}
\author[1]{Ernesto De Vito\thanks{Email: \texttt{devito@dima.unige.it}}}
\author[2]{\v Zeljko Kereta\thanks{Email: \texttt{zeljko@simula.no}}}
\author[3]{Valeriya Naumova\thanks{Email: \texttt{valeriya@simula.no}}}

\affil[1]{Universit\`a di Genova, Italy}
\affil[2,3]{Simula Research Laboratory, Norway}

\begin{document}
\setlength{\belowdisplayskip}{3pt} \setlength{\belowdisplayshortskip}{3pt}
\setlength{\abovedisplayskip}{3pt} \setlength{\abovedisplayshortskip}{3pt}

\title{Unsupervised parameter selection for denoising with the elastic net}

\maketitle

\begin{abstract}
  Despite recent advances in regularization theory, the issue of
  parameter selection still remains a challenge for most
  applications.  In a recent work the framework of
  statistical learning was used to approximate the optimal Tikhonov
  regularization parameter from noisy data.  In this work, we improve
  their results and extend the analysis to the elastic net regularization.  Furthermore, we design
  a data-driven, automated algorithm for the computation of an approximate regularization parameter.
  Our analysis combines statistical
  learning theory with insights from regularization theory. 
  We compare our approach with state-of-the-art parameter selection criteria and show that it has superior accuracy.
\end{abstract}

\tableofcontents
\section{Introduction}

Inverse problems deal with the recovery of an unknown quantity of interest
$\xbs\in\bbR^d$ from a corrupted observation $\ybs\in\bbR^m$. 
In most cases the relationship between $\xbs$ and $\ybs$ is linear, and can be approximately
described by 
\begin{equation}\label{eqn:inverse_problem}
\ybs = \Amat \xbs +\sigma\wbs,
\end{equation}
where $\Amat\in\bbR^{m\times d}$ is a known linear forward operator, $\wbs$ is a zero-mean isotropic random
vector, modeling the noise, and $\sigma>0$ is the noise level. 
Inverse problems of this type are ubiquitous in image processing, compressed sensing and other scientific fields.
In image processing applications they model tasks such as: denoising, where $\Amat$ is the identity; deblurring, where $\Amat$ is a convolution operator; and inpainting, where $\Amat$ is a masking operator.

The recovery of the original signal $\xbs$ from the corrupted observation $\ybs$ is an ill-posed inverse problem.
Thus, theory of inverse problems suggests the use of suitable regularization techniques \cite{zbMATH00936298}.
Specifically, in case of Gaussian noise, $\xbs$ is approximated with the minimizer of a regularized functional
  \begin{equation}\label{eqn:regularizer}
\argmin_{ \zbs\in\bbR^d} \N{\Amat \zbs- \ybs}_2^2 + \lambda
J(\zbs),
\end{equation}
where $\N{\cdot}_2$ is the Euclidean norm modeling \emph{data-fidelity}, $J$ is a \emph{penalty term}
encoding an \emph{a priori} knowledge on the nature of the true solution, and $\lambda$ is a \emph{regularization
parameter} determining a trade-off between these two terms. 
Having the penalty term fixed, a central issue concerns the selection of $\lambda$. 
The optimal parameter $\lambda$ is the one
that minimizes the discrepancy between the minimizer $\zbs^\lambda$
of~\eqref{eqn:regularizer} and
 the exact solution~$\xbs$
 \begin{equation}\label{eqn:lambda_ast}
\lambda_{\text{opt}}  = \argmin_{\lambda \in (0,+\infty)} \N{\zbs^\lambda
  -\xbs}_2.
\end{equation}
In the context of (semi-)supervised machine learning methods, regularization parameters are selected by evaluating \eqref{eqn:lambda_ast}, or another metric, over a training set of clean signals.
Unfortunately though, due to the curse of dimensionality accurate estimation of high-dimensional functions requires a number of samples that scales exponentially with
the ambient dimension. A common approach to mitigating these effects is to assume that the relevant data are supported on structures of substantially lower dimensionality.
On the other hand, in regularization theory  the clean image is unknown and hence $\lambda_\text{opt}$ is approximated using prior knowledge about the noise, such as the noise level.
Moreover, classical regularization theory is mostly concerned with the case when the data belongs to a function space (and is thus infinite dimensional).
In this case most existing parameter selection methods focus on the recovery of the minimum least-squares norm solution.
On the other hand, our work considers finite dimensional problems that incorporate additional constraints on the recovered solution in order to ensure it has the desired structure.
In many applications $\xbs$ is unknown (thus we cannot use supervised methods) and there is no available information about the noise $\wbs$ or the noise level $\sigma$.
Hence, $\lambda_{\text{opt}}$ needs to be approximated. 
Moreover, the lower level problem \eqref{eqn:lambda_ast} is often non-convex, even when \eqref{eqn:regularizer} is.

Choosing a good approximation to $\lambda_{\text{opt}}$ is a non-trivial, problem-dependent task that has motivated significant amounts of research over the last decades. 
However, there is still no framework that allows a fast and efficient parameter selection, particularly in a completely unsupervised setting.
In this paper, we aim at (partially) closing this gap and provide a novel concept for automated parameter selection by recasting the problem to the framework of statistical learning theory. 
Specifically, inspired by recent and (to our knowledge) first results in this spirit \cite{DFN16} we propose a method for learning the optimal parameter for elastic net regularization that uses a dimension reduction preprocessing step.
We emphasize that the method is unsupervised and requires minimal human interference.

\paragraph{Existing parameter selection methods.} 
Parameter selection rules used in regularization theory can be broadly classified as those based on the discrepancy principle \cite{morozov2012,AR2010}, generalized cross-validation (GCV)\cite{golub79}, balancing principle \cite{lepskii1991problem,DeVito2010}, quasi-optimality \cite{tikhonov1977,hofmann1986} and various estimations of the mean-squared error (MSE) (see \cite{stein1981,deledalle2014stein} and references therein).
GCV is a particularly popular parameter rule for linear methods since it gives a closed form for the regularization parameter and does not require tuning of any additional parameters or the knowledge of the noise. 
In specialized cases GCV can be extended to nonlinear problems \cite{wood2000}, but the regularization parameter is no longer given in closed form nor through an implicit equation. 
Balancing principle is a stable method that has received a lot of attention in the inverse problems community and has also been studied in the framework of learning theory, but requires tuning of additional parameters. 
Quasi-optimality is one of the simplest parameter choice methods. It does not require any information about the problem but it is not as stable as the balancing principle.
Discrepancy and MSE-based principles still remain the preferred methods for parameter selection for nonlinear estimators due to their simplicity and accuracy.
We refer to a recent rather comprehensive comparative study on the existing approaches \cite{zbMATH05929140}. 

In order to select the regularization parameter most existing methods require the regularized solution $\zbs^\lambda$ to be computed over a predefined grid of values of $\lambda$. 
The regularization parameters are then chosen according to some criteria, \emph{e.g.}, loss over a validation set.
To find regularization parameters by an exhaustive search is a computationally expensive task, especially in the high-dimensional data scenario, with often no guarantees on the quality of approximation.  
Moreover, most criteria presuppose that some \emph{a priori} information is  available, such as an accurate estimate of the noise level (in \emph{e.g.} discrepancy principle) or bounds on the noise error (in \emph{e.g.} balancing principle) and require additional, method-specific parameters to be preselected.

The main motivation of this work is to compute an accurate regularized solution $\zbs^\lambda$, with a nearly optimal $\lambda$, while ensuring low computational complexity and minimizing the need for manual intervention. 
In particular, we propose an unsupervised parameter selection method by recasting the problem to the framework of statistical learning, where we are interested in learning a function $\ybs \mapsto \widehat \lambda_{\text{opt}}$, from a training set of corrupted data, while ensuring that $\widehat\lambda_{\text{opt}}$ is a good approximation of the optimal parameter $\lambda_{\text{opt}}.$ 

\paragraph{Elastic net regularization.} 

Elastic net regularization was proposed by Zou and Hastie \cite{ZH05}, as
  \begin{equation}
  \zbs^{\lambda}(\ybs) = \argmin_\zbs \N{\Amat\zbs-\ybs}^2 +
        \lambda \LRP{\N{\zbs}_1 + \alpha\N{\zbs}_2^2},
        \label{eq:3}     
      \end{equation}
where  $\alpha\geq 0 $ is a hyperparameter controlling the
trade-off between $\ell_1$ and $\ell_2$ penalty terms.
Our main motivation for considering the elastic net is that it produces sparse models comparable to the Lasso (and is thus well suited for problems with data on lower dimensional structures), while often achieving superior accuracy in real-world and simulated data.
Moreover, the elastic net overcomes the main limitations of $\ell_1$ minimization.
Namely, it encourages the grouping effect, which is relevant for many real-life applications such as microarray classification and face detection (see \cite{DDR09} and references therein). 

To solve \eqref{eq:3}  the authors in \cite{ZH05} rewrite the elastic net functional as Lasso regularization with augmented datum, use LARS \cite{efron2004} to reconstruct the entire solution path, and apply cross-validation to select the optimal regularization parameter. 
Later work \cite{DDR09} studies theoretical properties of \eqref{eq:3} in the context of learning theory, analyzes the 
underlying functional and uses \emph{iterative soft-thresholding} to compute the solution.  
For the parameter choice the authors provide an adaptive version of the balancing principle \cite{lepskii1991problem,DeVito2010}. 
The rule aims to balance approximation and sample errors, which have contrasting behavior with respect to the tuning parameter, but requires (potentially) many evaluations of $\zbs^\lambda$.
We will rework some of the arguments from \cite{DDR09} for the computation of $\zbs^{\lambda}$, while keeping our focus on an efficient approach for parameter learning.
In \cite{JLS09} the authors propose an \emph{active set} algorithm for solving \eqref{eq:3}.  
Addressing the problem in the framework of classical regularization theory, the authors consider the discrepancy principle  \cite{morozov2012,bohesky2009} for determining the parameter. 
This requires estimations of the solution $\zbs^{\lambda}$ for many parameter values, and a pre-tuning of other, method-specific parameters. 
Moreover, it is assumed that the noise level is known, which is often not the case in practice. 

The authors in \cite{LR10} use a hybrid alternating method for tuning parameters $\lambda$ and $\alpha$ for the model fitting problem
$
\ysf_i = \abs_i^\top \xbs, i = 1,\ldots, n, 
$,
where $\ysf_i \in \bbR,$ $\abs_i \in \bbR^p$ and $\xbs \in \bbR^p.$ 
First step is to update the solution $\zbs^{\lambda}$, using coordinate descent, and then to update $\lambda$ and $\alpha$ in one iteration. 
The main advantage is the efficiency, as one does not need to calculate $\zbs^{\lambda}$ for multiple parameters at once, but rather on a much coarser parameter grid. 
The method is in spirit similar to LARS, but has better scalability. It requires that a non-convex problem is solved, and hence has inherent limitations. 
Moreover, it cannot be used in the setting of general inverse problems \ref{eqn:inverse_problem}, where the design matrix $\Amat$ is fixed and each response $\ybs$ is generated by a new clean signal $\xbs$.
{In summary, we are not aware of any parameter selection rule for the elastic net that allows to select $\lambda$ without \emph{a priori} assumptions and without extensive manual adjustments.}

This paper leverages the work \cite{DFN16} where the parameter selection is considered in the context of non-parametric regression with random design. 
In particular, the authors propose a data-driven method for determining the optimal parameter for Tikhonov regularization, under the assumption that
a training set of independent observations
\( \ybs_1,\ldots, \ybs_N\) is made available,
each of them associated with an (unknown) signal
$\xbs_1,\ldots, \xbs_N$ through
\(
\ybs_i = \Amat \xbs_i +\sigma \wbs_i.
\)
The starting point of the method is to find an empirical proxy $\widehat\xbs $ of the real solution $\xbs$ by assuming that $\xbs_1,\ldots, \xbs_N$ are distributed over a lower-dimensional linear subspace and then select the regularization parameter as
  \begin{equation}
\widehat{\lambda}_{\text{opt}} = \argmin_{\lambda \in (0,+\infty)} \N{\zbs_{\text{Tik}}^\lambda -\widehat\xbs}_2,\label{eq:2}
\end{equation}
where $\zbs_{\text{Tik}}^\lambda$ is the minimizer of the Tikhonov
functional 
\(
\min_{\zbs\in\bbR^d} \N{\Amat\zbs-\ybs}_2^2 +
        \lambda \N{\zbs}_2^2.
\)
The analysis and techniques related to $\widehat\xbs$ are independent of the choice of the optimization scheme, whereas the selection of $\widehat{\lambda}_{\text{opt}}$ is defined by the regularization scheme. 
However, it is worthwhile to mention that if $\Amat^\dagger$ is not injective, $\widehat{\xbs}$ is not a good proxy of $\xbs$. For Tikhonov regularization this is not an issue as, without loss of generality, we can always assume that $\Amat$ is injective. 
Specifically, one can replace  $\xbs$ in~\eqref{eqn:lambda_ast} with $\xbs^\dagger= \Amat^\dagger \Amat\xbs$ and recall that $\zbs_{\text{Tik}}^\lambda$ belongs to $\Null \Amat^{\perp}$ for all $\lambda$. 
Therefore, for wider applicability of the suggested framework, it is important to address the selection of $\widehat{\lambda}_{\text{opt}}$ for a wider class of regularizers and inverse problems.

In this paper, we extend the framework of \cite{DFN16} by providing the analysis for the elastic net regularization and improving the theoretical results.
Moreover, we develop an efficient, fully automated algorithm that is extensively tested on synthetic and real-world examples.
The last point is the main practical contribution of our paper. 
Namely, our goal is not to introduce a new regularization paradigm but rather to design a fast and unsupervised  method for determining a near optimal regularization parameter for existing regularization methods.
To do this, we analyze our problem in {two settings}:
\begin{enumerate}
\item {\it simplified case $\Amat = \Id$} (corresponding to image denoising):
We restate the lower level problem and show that in case of a bounded $\wbs$ it admits a unique minimizer, which motivates our algorithm.
Furthermore, we provide a bound on
$| \lambda_{\text{opt}} - \widehat{\lambda}_{\text{opt}}|$ for independent Bernoulli random noise and discuss the number of samples needed for optimal learning, see Proposition~\ref{prop:error_estimate}. 
Though the latter model might be oversimplified, it captures the essence of the problem and our experiments confirm the results in more general settings. 

\item {\it general case:} for a general matrix $\Amat$ we provide an unsupervised, efficient, and accurate algorithm for the computation of an approximate optimal parameter. 
We study the performance of our algorithm, comparing it to state-of-the-art parameter choice methods on synthetic and image denoising problems. The obtained results show that our approach achieves superior accuracy.
\end{enumerate}

\subsection{Outline}

In Section \ref{sec:setting}, we describe the main ingredients of our approach. 
We define and prove bounds regarding {empirical estimators}, discuss minimizers of the elastic net \eqref{eq:3}, and define loss functions that will be used for parameter selection. 
Section \ref{sec:error_analysis} provides the main theoretical results of the paper regarding loss functionals and their minimizers.
In Section \ref{sec:general_matrices} we present an efficient and accurate algorithm for the computation of an approximate optimal parameter.
We study the performance of our method through several numerical experiments on synthetic and imaging data in Section \ref{sec:experiments}. Therein, we  compare our method with state-of-the-art parameter selection criteria in terms of accuracy of the solution recovery, closeness to the optimal parameter, sparse recovery and computational time.
For imaging tasks our focus is on wavelet-based denoising where we work on synthetic images and real-world brain MRIs.
We conclude with a brief discussion about future directions in Section \ref{sec:conclusion}.
The Appendix contains proofs of auxiliary results.

\subsection{Notation}
The Euclidean and the $\ell_1$-norms of $\ubs=(\usf_1,\ldots,\usf_d)^\top$ are denoted by $\N{\ubs}_2$ and
$\N{\ubs}_1$, respectively. The modulus function $\SN{\cdot}$, the \emph{sign}
function  $\Sgn(\cdot)$, and the \emph{positive part} function
$(\cdot)_+$ are defined component-wise for $i=1,\ldots,d,$ by
\(
\SN{\ubs}_i =\SN{\usf_i},\) \(\Sgn(\ubs)_i = \Sgn(\usf_i),\) and \(
((\ubs)_+)_i= (\usf_i)_+, 
\)
where for any $\usf\in\bbR$
\[\Sgn(\usf)=
\begin{cases} 1,&\quad \textrm{if } \usf>0,\\
 0,&\quad \textrm{if } \usf=0,\\-1,&\quad \textrm{if }
\usf<0,\end{cases}\quad\textrm{and}\quad(\usf)_+=\max\{0,\usf\}.
\]
The canonical basis of $\bbR^d$ is denoted by $\{\ebs_i\}_{i=1,\ldots,d}$.
We denote the transpose of a matrix $\Mmat $ by $\Mmat^\top,$ the Moore-Penrose pseudo inverse by $\Mmat^\dagger$, and the spectral norm by $\N{\Mmat}_2$.
Furthermore, $\range(\Mmat)$ and $\Null(\Mmat)$ are the range and the null space of $\Mmat,$ respectively. For a square-matrix $\Mmat$, we use $\trace (\Mmat)$ to denote its trace.
The identity matrix is denoted by $\Id$ and we use $\mathbbm{1}_{\CD}$ for the indicator function of a set $\CD\subset\bbR^d$.
For any $\vbs\in\bbR^d$,
$\vbs\otimes \vbs$ is the rank one operator acting on $\wbs\in\bbR^d$ as
$(\vbs^\top \wbs) \vbs$. 

A random vector $\xibf$ is called \emph{sub-Gaussian} if
\[\N{\xibf}_{\psi_2} :=\sup_{\N{\vbs}=1}\sup_{q\geq 1} q^{-\frac{1}{2}} \bbE\LRS{\SN{\vbs^\top \xibf}^q}^{\frac{1}{q}}  < +\infty.\]
The value $\N{\xibf}_{\psi_2}$ is the sub-Gaussian norm of $\xibf$, with which the space of sub-Gaussian vectors becomes a normed vector space \cite{V18}. 
The (non-centered) covariance of a random vector $\xibf$ is denoted as
\[\Sigma({\xibf}):=\cov{\xibf} = \bbE [\xibf \otimes \xibf].\]
We write $a \lesssim b$ if there exists an absolute constant $C>0$ such that $a \leq Cb.$

\section{Problem setting}\label{sec:setting}
We consider the following stochastic linear inverse problem:
given a deterministic matrix  $\Amat\in\bbR^{m\times d}$, we are interested in recovering a vector $\xbs\in\bbR^d$ from a noisy observation $\ybs\in\bbR^m$ obeying
\begin{equation}\label{eqn:LinearModel}
\ybs = \Amat \xbs +\sigma\wbs, 
  \end{equation}
where 
\begin{enumerate}[label=(\textsf{A}\arabic*)]
  \item\label{ass:x} the unknown datum $\xbs\in\bbR^d$ is a sub-Gaussian vector, such that $\N{\xbs}_{\psi_2}=1$;
  \item\label{ass:V}  there exists a subfamily
            $1\leq i_1<\ldots i_h\leq d$ of $h$ indices, with $h\ll d$, such that
\[ \CV:=\range\LRP{\Sigma(\xbs)}=\Span{\ebs_{i_1},\ldots,\ebs_{i_h}}\]
and $\Null(\Amat)\cap\CV=\{\boldsymbol0\}$;
  \item\label{ass:noise} the noise $\wbs\in\bbR^m$ is an
          independent sub-Gaussian vector, such that
          $\N{\sigma\wbs}_{\psi_2}\leq 1$, $\Sigma({\wbs})=\Id$ and
          $\sigma>0$ is the noise level. 
\end{enumerate}
Conditions~\ref{ass:x} and~\ref{ass:noise} are standard assumptions on the distributions of the exact datum $\xbs$ and the noise $\sigma\wbs$, ensuring that the tails have fast decay.
Note also that normalization conditions on  $\xbs$ and $\wbs$ can always be satisfied by rescaling.  
Furthermore, it follows from the definition that $\CV$ is the smallest subspace such that $\xbs\in \CV$, almost surely. 
Thus, by~\ref{ass:V}, the exact datum $\xbs$ is almost surely $h$-sparse and since
$\Null\LRP{\Amat}\cap \CV=\{0\}$, it is a unique vector with that property. 
Define now
\( \CW=\range\LRP{\Sigma({\Amat \xbs})}.\) 
The following simple result  was shown for an injective $\Amat$ in \cite{DFN16}; here we extend it to the general case.

\begin{lemma}\label{lem:AV}
Under~Assumption~\ref{ass:V} we have $\dim\CW=h$ and $\CW =\Amat\CV$.
\end{lemma}
\begin{proof}
A direct computation gives
  \[\Sigma({\Amat\xbs})=\bbE[\Amat\xbs\otimes\Amat\xbs]=\Amat\Sigma({\xbs})\Amat^\top = \Amat\Pmat
          \Sigma({\xbs}) (\Amat\Pmat)^\top,\]
where $\Pmat$ denotes the orthogonal projection onto $\CV$.  
Assumption \ref{ass:V} says that $\Amat$ is injective on $\CV$, and thus $\Sigma({\Amat\xbs})$ and $\Sigma({\xbs})$ have the same rank. 
Furthermore $(\Amat\Pmat)^\top$ maps $\bbR^d$ onto $\CV$, so that
\[ \range (\Sigma({\Amat\xbs})) = (\Amat \Pmat \Sigma({\xbs})) \CV=\Amat\Pmat \CV=\CW,\]
where $\Sigma({\xbs})\CV=\CV$, since $\Sigma({\xbs})$ is symmetric.
\end{proof}
\subsection{Empirical estimators} 
Lemma \ref{lem:AV} suggests that $\CV$ could be directly recovered if $\Amat$ were invertible and $\CW$ were known.
In most practical situations though, neither of those assumptions is satisfied: we only have access to noisy observations and $\Amat$ could not only be non-invertible, but also non-injective. 
We will address this issue by recasting the problem to a statistical learning framework, similar to \cite{DFN16}.
Namely, suppose we are given observation samples $\ybs_1,\ldots,\ybs_N$ such that $\ybs_i=\Amat\xbs_i+\sigma\wbs_i$ for $i=1,\ldots,N$, where
$\xbs_i,\wbs_i$ and $\sigma$ are unknown, and
let \[\widehat\Sigma(\ybs)= \frac{1}{N}\sum_{i=1}^N
\ybs_i\otimes\ybs_i\]
be the empirical covariance of the observations. 
Standard statistical learning theory suggests that $\widehat{\Sigma}\LRP{\ybs}$ is a good approximation to $\Sigma\LRP{\ybs}$ provided $N$ is large enough.
As a consequence, we will show that a vector space spanned by the first $h$ eigenvectors of $\widehat{\Sigma}\LRP{\ybs}$, denoted by $\widehat{\CW}$, is a good estimator of $\CW$. 

To justify the above claims, observe first that since $\Sigma(\wbs) = \Id$ holds by~\ref{ass:noise}, we have
\begin{equation}\label{eqn:SigmaAx}\Sigma\LRP{\ybs}  = \Sigma\LRP{\Amat\xbs} + \sigma^2\Id.\end{equation}
Therefore, $\Sigma\LRP{\ybs}$ and $\Sigma\LRP{\Amat\xbs}$ have the same eigenvectors and the spectrum of $\Sigma\LRP{\ybs}$ is just a
shift of the spectrum of $\Sigma\LRP{\Amat\xbs}$ by $\sigma^2$. 
Let $\lambda_1\geq\ldots\geq\lambda_h$ be the non-zero
eigenvalues of $\Sigma(\Amat\xbs)$, counting for multiplicity, and
$\alpha_1\geq\ldots\geq\alpha_m$ and
  $\widehat\alpha_1\geq\ldots\geq\widehat\alpha_m$ be the eigenvalues
of $\Sigma(\ybs)$ and $\widehat\Sigma(\ybs)$, respectively. 
From \eqref{eqn:SigmaAx} it follows
\begin{equation}
  \label{eq:1}
  \begin{cases}
    \alpha_i=\lambda_i+\sigma^2, & \text{for }  i=1,\ldots, h,\\
\alpha_i=\sigma^2,  &  \text{for }i=h+1,\ldots,m
  \end{cases}.
\end{equation}
Let $\Pi$ be the (orthogonal) projection onto $\CW$, which has rank
$h$ due to Lemma~\ref{lem:AV}, and let $\widehat{\Pi}$ be
the (orthogonal) projection onto $\widehat{\CW}$. 
We now show the fundamental tool of our study: that $\widehat{\Pi}$ is an accurate and an unbiased approximation of $\Pi$.
We distinguish between bounded and unbounded $\ybs$ and improve upon results in \cite{DFN16}.

\begin{lemma}\label{lem:proj_error}
Assume that $\sigma^2<\lambda_h$. Given $u>0$, with probability greater than $1-2\exp(-u)$
  \begin{equation}\label{eqn:unbounded}
\N{\widehat{\Pi}-\Pi}_2 \lesssim \frac{\lambda_1}{\lambda_h}
\LRP{\sqrt{\frac{{h+\sigma^2m +u}}{N}} +
  \frac{{h+\sigma^2m +u}}{N} },
\end{equation}
provided $N\gtrsim \LRP{h+\sigma^2m+u}$.  Furthermore, if $\ybs$ is
bounded, then with probability greater than $1 - \exp(-u)$
\begin{equation}\label{eqn:bounded}
\N{\widehat{\Pi}-\Pi}_2 \lesssim \frac{\lambda_1}{\lambda_h}
\LRP{\sqrt{\frac{{\log(h+m)+u }}{N}} +  \frac{\log(h+m)+u  }{N} },
\end{equation}
provided $N\gtrsim (\log(2m)+u)$.
\end{lemma}
\begin{proof}
We will first show \eqref{eqn:unbounded}.
Using Theorem 9.2.4 and Exercise 9.2.5 in \cite{V18}, we have
\begin{equation}\label{eqn:low_rank_cov_estimation}
  \N{\Sigma(\ybs)-\widehat{\Sigma}(\ybs)}_2 \lesssim
  \N{\Sigma(\ybs)}_2\LRP{\sqrt{\frac{{r+u}}{N}} +
    \frac{{r+u}}{N}},\end{equation}
with probability greater than $1-2\exp(-u)$,  where
$r=\trace\LRP{\Sigma(\ybs)}/\N{\Sigma(\ybs)}_2$ is the \emph{stable rank} of $\Sigma(\ybs)$. 
Using \(\trace\LRP{\Sigma(\ybs)}\leq \lambda_1 h + m\sigma^2\)
and  $\N{\Sigma(\ybs)}_2=\lambda_1+\sigma^2\leq 2\lambda_1$, we get
  \begin{equation}\label{eqn:low_rank_writtenout}
\N{\Sigma(\ybs)-\widehat{\Sigma}(\ybs)}_2 \lesssim 
\lambda_1
\LRP{\sqrt{\frac{{h+\sigma^2m+u} }{N}} + \frac{{h+\sigma^2m+u }}{N}}.
\end{equation}

Let $\alpha^*=\alpha_h>\sigma^2$. By~\eqref{eq:1} it follows that $\Pi$ is the
projection onto the linear span of those eigenvectors of $\Sigma(\ybs)$
whose corresponding eigenvalue is greater than or equal to $\alpha^*$.  Using
\(
\alpha_{h}-\alpha_{h+1} = \lambda_h
\), by \eqref{eqn:low_rank_writtenout} we have
\begin{equation}\label{eqn:N_requirement} \epsilon:=\N{\Sigma(\ybs)-\widehat\Sigma(\ybs)}_2 < 
  \frac{\alpha_h -\alpha_{h+1}}{2}=\frac{\lambda_h}{2},\end{equation}
provided $N\gtrsim \LRP{h+\sigma^2m +u}$.
Let now $\Pi_{\alpha^*}$ be the projection onto the linear span of those eigenvectors of $\widehat\Sigma(\ybs)$
whose corresponding eigenvalue is greater than or equal to $\alpha^*$. 
As a consequence of Theorem 7.3.1 in \cite{bhatia97}, there exists an
eigenvalue $\widehat{\alpha}^*$ of $\widehat\Sigma(\ybs)$ such that
\begin{alignat}{1}
 |\alpha^*-\widehat{\alpha}^*|&\leq \epsilon, \text{ and } \operatorname{dim} \Pi_{\alpha^*} = \operatorname{dim}\Pi \label{eq:6c} \\
 \widehat{\alpha}_j &\leq \alpha_{h+1} + \epsilon =\sigma^2 +\epsilon, \quad \forall
 \widehat{\alpha}_j<\widehat{\alpha}^* \label{eq:6b} \\
\N{\Pi_{\alpha^*}-\Pi}_2 & \leq \frac{1}{\lambda_h-\epsilon}
\N{(\Id-\Pi_{\alpha^*})(\widehat\Sigma(\ybs)-\Sigma(\ybs))\Pi}_2 \label{eq:6d}.
\end{alignat}
By~\eqref{eq:6c} it follows that
$\widehat{\alpha}^*=\widehat{\alpha}_h$
so that $\Pi_{\alpha^*}=\widehat{\Pi}$ and hence
\begin{equation}
  \label{eq:7}
  \N{\widehat{\Pi}-\Pi}_2  \leq \frac{1}{\lambda_h-\epsilon}
\N{\widehat\Sigma(\ybs)\Pi-\Sigma(\ybs)\Pi}_2 \leq \frac{2}{\lambda_h} \N{\widehat\Sigma(\ybs)\Pi-\Sigma(\ybs)\Pi}_2 .
\end{equation}
Since $\N{\widehat\Sigma(\ybs)\Pi-\Sigma(\ybs)\Pi}_2\leq
\N{\widehat\Sigma(\ybs)-\Sigma(\ybs)}_2$, the claim follows by \eqref{eqn:low_rank_cov_estimation}.

Assume now that $\N{\ybs}_2 \leq \sqrt{L}$ holds almost surely and consider a family of independent $m\times
h$ matrices
\[
\Smat_i = {\ybs_i^\top \ybs_i \Pi - \Sigma(\ybs)\Pi},\qquad i=1,\ldots,N.
\]
Since $\frac{1}{N}\sum_{i=1}^N \Smat_i= \widehat\Sigma(\ybs)\Pi-\Sigma(\ybs)\Pi$ we can apply the matrix Bernstein inequality for rectangular matrices (Theorem 6.1.1. in \cite{Tro15}). Thus, for  $u>0$ we have
\[
\bbP\left(\N{\widehat\Sigma(\ybs)\Pi-\Sigma(\ybs)\Pi}_2\geq u
\right)\leq (m+h) \exp\LRP{\dfrac{-Ns^2 }{M+ 2L u/3}}
\]
where $M>0$ is a matrix variance constant independent of $m$, $h$, and $d$, such that
\begin{alignat*}{1}
  \max\left\{\bbE \N{\Smat_i^\top \Smat_i}_2,\bbE \N{\Smat_i\Smat_i^\top}_2\right\}  \leq M.
\end{alignat*}
A direction computation gives $M\leq L\N{\Sigma(\ybs)}_2$.
It follows that 
\[
\N{\widehat\Sigma(\ybs)\Pi-\Sigma(\ybs)\Pi}_2 \lesssim \lambda_1
\LRP{\sqrt{\frac{{\log(h+m)+u }}{N}} +  \frac{\log(h+m)+u  }{N} },
\]
holds with probability greater than $1-\exp(-u)$ for every $u>0$.
Moreover, by analogous argumentation \eqref{eqn:N_requirement} holds provided $N\gtrsim (\log(2m) + u)$, see \ref{sec:bounded_y_covs} for details. 
Thus, \eqref{eqn:bounded} follows by applying \eqref{eq:7}. 
\end{proof}
The previous result comes with a certain caveat.  
Namely, the proof implicitly assumes that either $h$ or the spectral gap are known (which informs the choice of the approximate projector $\widehat{\Pi}$).  
In practice however, the desired eigenspace can only be detected if there is a spectral gap and if it corresponds to the eigenspace we want to recover,  {\em i.e.},   if $ \lambda_h>\delta$, where
\begin{equation}\label{eqn:SPGAP}
\delta =\max_{i=1,\ldots,h-1}\LRP{\lambda_i-\lambda_{i+1}}=\max_{i=1,\ldots,h-1}\LRP{\alpha_i-\alpha_{i+1}}.
\end{equation}
\begin{proposition}\label{prop:sp_gap}
Assume \eqref{eqn:SPGAP} holds. Then the empirical covariance matrix has a spectral gap at the $h\textrm{-th}$ eigenvalue, with probability greater than $1-2\exp(-u)$, provided $\delta<\lambda_h$ and $N\gtrsim \frac{\lambda_1^2}{\LRP{\lambda_h-\delta }^2}(h+u)$.
\end{proposition}
\begin{proof}
Assume $\N{\Sigma(\ybs)-\widehat\Sigma(\ybs)}_2 < \epsilon$ holds for $\epsilon >0$.
Since $\sup_{i=1,\ldots, m}\SN{\widehat{\alpha}_j-{\alpha}_{j}} \leq \N{\Sigma(\ybs)-\widehat\Sigma(\ybs)}$ we get
\[ \SN{\widehat{\alpha}_j-\widehat{\alpha}_{j+1}} \leq 2\epsilon + \SN{{\alpha}_j-{\alpha}_{j+1}},\] 
by adding and subtracting $\alpha_j$ and $\alpha_{j+1}$ inside the first term.
Thus, if $j>h$ then $\SN{\widehat{\alpha}_j-\widehat{\alpha}_{j+1}} \leq 2\epsilon$, and if $j<h$ then $\SN{\widehat{\alpha}_j-\widehat{\alpha}_{j+1}} \leq 2\epsilon + \delta $.
For $j=h$ on the other hand we have $\SN{\widehat{\alpha}_h-\widehat{\alpha}_{h+1}} > \SN{{\alpha}_h-{\alpha}_{h+1}} - 2\epsilon$.
In conclusion, 
\[\argmax_{i=1,\ldots,m-1} \LRP{\widehat \alpha_i - \widehat \alpha_{i-1}}
= h\]
holds provided
provided $\epsilon<\frac{\lambda_h-\delta }{4}$.
Using \eqref{eqn:low_rank_cov_estimation} the claim follows.
\end{proof}
It is clear that if $\delta>\lambda_h$ the  spectral gap of the empirical covariance matrix
is at \(\argmax_{i=1,\ldots,h-1}\LRP{\lambda_i-\lambda_{i+1}}\), which is smaller than $h$.  
In practice though, the situation is not as pessimistic as this observation would suggest and we can rely on a wealth of \emph{ad hoc} remedies.
We devote more attention to this question in Section \ref{sec:estimating_sparsity}, and suggest alternative heuristics for estimating the intrinsic dimension $h$.

We are ready to define our {empirical estimator} of $\xbs$. 
Let $\Qmat = \Amat\Amat^{\dagger}$ be the orthogonal projection onto $\Range(\Amat)=\Null^\perp(\Amat^\top)$, and $\Pmat=\Amat^{\dagger}\Amat$ the orthogonal projection onto $\Range(\Amat^\top)=\Null^\perp(\Amat)$.
The \emph{{empirical estimator}} of $\xbs$ is defined as
\begin{equation}\label{eqn:xhat}
  \widehat{\xbs}=\Amat^\dagger\widehat{\Pi}\ybs.  
\end{equation}
For \( \widehat{\boldsymbol{\Veta}} = \ybs - \widehat{\Pi}\ybs\) the {empirical estimator} $\widehat{\xbs}$  satisfies the (empirical) inverse problem
\begin{equation}
\label{eqn:empiricalprob}\Amat\widehat{\xbs} +
    \Qmat\widehat{\Veta} = \Qmat\ybs.
\end{equation}

A direct consequence of \eqref{eqn:empiricalprob} is that minimizers of the empirical and of the original problem coincide.
\begin{lemma}
  Let
  \(\widehat{\zbs}^\lambda{(\ybs)}=\argmin_{\zbs} \N{\Amat\zbs-\Qmat\ybs}^2 + \lambda\, J(\zbs).\) 
  Then $\widehat{\zbs}^\lambda{(\ybs)}=\zbs^\lambda(\ybs)$.
\end{lemma}
\begin{proof}
  We  compute
\(  \N{\Amat \zbs - \ybs}_2^2  =\N{\Amat\zbs -\Qmat\ybs + (\Qmat-\Id)\ybs}_2^2.\)
  Since $\Qmat$ is an orthogonal projection onto $\Range(\Amat)$ it follows $(\Qmat-\Id)\ybs\in\Range^\perp(\Amat)$. Using Pythagoras' theorem we thus have
\( \N{\Amat \zbs - \ybs}_2^2 = \N{\Amat\zbs -\Qmat\ybs}_2^2 + \N{(\Qmat-\Id)\ybs}_2^2.\)
Since the second term does not depend on $\zbs$ we get
\[ \argmin_\zbs\N{\Amat\zbs-\ybs}^2_2 +\lambda J(\zbs) = \argmin_\zbs \N{\Amat\zbs-\Qmat\ybs}^2_2 +\lambda\,J(\zbs).\]
\end{proof}

The definition of $\widehat{\xbs}$ is independent of the choice of the optimization scheme.
In the following, we use $\widehat{\xbs}$ to learn a nearly optimal regularization parameter for elastic net minimization.
Before doing so, we note that one might want to consider $\widehat{\xbs}$ as an approximate solution by itself, and completely avoid regularization and thus the issue of parameter choice. 
Experimental evidence in Section \ref{sec:experiments} shows that when $\Amat$ is not injective, the training set size $N$ is small, or when the noise level is small, $\|  \widehat\xbs - \xbs\|$ is larger than $\|  \zbs^\lambda - \xbs\|,$ some of which has also been observed in \cite{DFN16}. 
In addition, $\widehat\xbs$ does not preserve the structure of the original signal, \emph{e.g.}, $\widehat\xbs$ will in general not be sparse for a sparse $\xbs$, and regularization is needed. 
Lastly, we remind that we are interested in using an estimator $\widehat{\xbs}$ for which $\SN{\lambda_{\text{opt}} -\widehat{\lambda}_{\text{opt}}}$ is small with high probability (\emph{i.e.} the one that can be used to derive an accurate parameter selection) and we are not interested in directly controlling $\N{\xbs- \widehat{\xbs}}$, which is the {goal in manifold learning} \cite{bel06}.

\subsection{Elastic net minimization}\label{sec:enets}

From now on we focus on the parameter choice for the elastic net, where  $J(\zbs) = \N{\zbs}_1 + \alpha\N{\zbs}_2^2$, 
so that 
\begin{equation}\label{eq:enet_reg}
  \zbs^{\lambda}(\ybs) = \argmin_{\zbs\in\bbR^m} \N{\Amat\zbs-\ybs}_2^2 + \lambda \LRP{\N{\zbs}_1 + \alpha\N{\zbs}_2^2}.
\end{equation} 
The term $\N{\zbs}_1$, enforces the sparsity of the solution, whereas $\N{\zbs}_2^2$ enforces smoothness and ensures that in case of highly correlated features we  can correctly retrieve all the relevant ones.
We first recall some  basic facts about existence, uniqueness and sensitivity of elastic net solutions with respect to regularization parameters \cite{JLS09}.
\begin{lemma}\label{lem:solutions}
The elastic net functional is strictly convex and coercive. 
Moreover, for each $\lambda>0$ the minimizer of \eqref{eqn:regularizer} exists, is unique and the mapping $\lambda\mapsto\zbs^\lambda$ is continuous with respect to $\lambda>0$. 
\end{lemma}

In the remainder of this paper we will recast \eqref{eq:enet_reg} as
\begin{align}\label{eqn:enets_minimization}
  \zbs^{t}(\ybs) & = \argmin_{\zbs\in\bbR^m} t\N{\Amat\zbs-\ybs}_2^2 + (1-t)\LRP{\N{\zbs}_1 + \alpha\N{\zbs}_2^2},
\end{align}
where $t\in[0,1]$ and $\alpha>0$ is a fixed parameter. 
For $t\in(0,1)$ the solutions of \eqref{eqn:enets_minimization} correspond to solutions of \eqref{eq:enet_reg} for $\lambda=\frac{1-t}{t}.$
On the other hand, for $t=0$ we get $\zbs^0(\ybs)=\boldsymbol{0}$, and for $t=1$ we define $\zbs^1(\ybs):=\xbs^\alpha,$ where 
\begin{equation}\label{eqn:xalpha}
\xbs^\alpha=\argmin_{\zbs\in\CN} \LRP{\N{\zbs}_1 + \alpha\N{\zbs}_2^2},
\text{ for } \CN=\{\zbs\in\bbR^m\mid \Amat^\top \Amat\zbs=\Amat^\top\ybs\}.
\end{equation} 
This definition is driven by the following observations.
First, the set $\CN= \Amat^\dagger \ybs \oplus \ker\LRP{\Amat}$ is non-empty (since $\Amat$ has finite rank). 
Furthermore, it was shown in \cite{DDR09} and \cite{JLS09} that in case of elastic nets minimization
    \begin{equation}
\lim_{t\to 1} \zbs^t = {\xbs}^\alpha\label{eq:4}
\end{equation} 
In other words, $\xbs^\alpha$ plays the role of the Moore-Penrose solution in linear regularization schemes \cite{zbMATH00936298}.
By Lemma \ref{lem:solutions} the minimizer of
\eqref{eqn:enets_minimization} always exist and is unique, the map $t\mapsto \zbs^t$ is continuous for $t\in(0,1)$. 
Equation ~\eqref{eq:4} implies that $t\mapsto \zbs^t$ is continuous at $t=1$, and later in \eqref{eq:5} we show that the continuity also holds at $t=0$.
\paragraph{Solution via soft-thresholding. } 
The elastic net does not admit a closed form solution in case of a general forward matrix $\Amat$. 
In Zou and Hastie \cite{ZH05} the elastic net problem is recast as a Lasso problem with augmented data, which can then be solved by many different algorithms (\emph{e.g.} the LARS method \cite{efron2004least}).
Alternative algorithms compute the elastic net minimizer directly, and are generally either of the \emph{active set}  \cite{JLS09} or the \emph{iterative soft-thresholding}-type  \cite{DDR09}.
Here we adhere to iterative soft-thresholding, and rework the arguments in \cite{DDR09} to show that the solution to \eqref{eqn:enets_minimization} can be obtained through fixed point iterations for all $t\in[0,1]$.  
To begin, define the soft-thresholding function by
\begin{align}\label{eqn:Stw} \CS_\tau(\usf) &= \Sgn(\usf)\LRP{\SN{\usf}-\frac{\tau}{2}}_+
\end{align}
and the corresponding soft-thresholding operator
$\boldsymbol{\CS}_\tau(\boldsymbol{\ubs})$, acting component-wise on vectors $\ubs\in\bbR^m$.
The next lemma is a direct reworking of the arguments in \cite{DDR09} and states that \eqref{eqn:enets_minimization} is a fixed point  of a contractive map.
\begin{lemma}\label{lem:contraction}
The solution to \eqref{eqn:enets_minimization}, for
$\Amat\in\bbR^{m\times d}, \ybs\in\bbR^m$ and $t\in (0,1)$, satisfies $\zbs= \CT_t(\zbs)$, where the map $\CT_t:\bbR^d\rightarrow\bbR^d$  is a contraction and is defined by
\begin{equation}\label{eqn:TtDefn}\zbs^t =\CT_t(\zbs) = \frac{1}{\tau
    t +(1-t)\alpha}\Csv_{1-t}\LRP{t\LRP{\theta\Id-\Amat^\top\Amat}\zbs+t\Amat^\top\ybs},
\end{equation}
with the Lipschitz constant 
\[
\dfrac{t (\sigma^2_M- \sigma_m^2) }{t (\sigma^2_M+ \sigma_m^2)
  +2\alpha (1-t)}<1, 
\]
where $\theta= \frac{\sigma_m^2+\sigma_M^2}{2}$,  and $\sigma_m$ and
$\sigma_M$ are the smallest and the largest singular values of the matrix~$\Amat$, respectively.
\end{lemma}
For $t=0$, the solution is $\zbs^0=\boldsymbol{0}$, which is consistent with~\eqref{eqn:TtDefn}. 
Furthermore, by~\eqref{eqn:Stw} and~\eqref{eqn:TtDefn}, we get
\begin{equation}
  \label{eq:5}
  \zbs^t =0 \qquad \text{ if }\quad 0\leq  t\leq \frac{1}{1+2 \N{\Amat^\top\ybs}}.
\end{equation}
Our definition  of the solution at $t=1$ in \eqref{eqn:xalpha} also satisfies $\zbs= \CT_t(\zbs)$ since  $\Csv_{1-t}$ is identity and thus
$\zbs^{t=1}(\ybs)=\Amat^\dagger\ybs$, though  $\CT_1$ is not a contraction. 
In summary, the solutions are consistent with Lemma \ref{lem:solutions}, as expected.

\paragraph{Closed form solution. } 
In the case of orthogonal design \cite{ZH05}, \emph{i.e.} $\Amat^\top\Amat=\Id$, the solution of \eqref{eqn:regularizer} is given by
\begin{equation}\label{eqn:Naive} \zbs^\lambda = \frac{(\lvert\Amat^\top\ybs\lvert -\lambda/2)_+}{1+\alpha\lambda}\Sgn(\Amat^\top\ybs).\end{equation}
Plugging $\lambda=\frac{1-t}{t}$ into (\ref{eqn:Naive}) we have
\begin{equation}\label{eqn:naive} \zbs^{t}(\ybs) =
  \frac{(t(1+2\lvert\Amat^\top\ybs\lvert)-1)_+}{2(t(1-\alpha)+\alpha)}\Sgn\LRP{\Amat^\top\ybs}.
\end{equation}

\subsection{Quadratic loss functionals}
\label{subsec:qlf}

To select the regularization parameters we go back to the first principles and consider quadratic loss functionals. 

\begin{definition}\label{defn:loss}
  Functions $R,\widehat{R}\colon[0,1]\rightarrow\bbR$, defined by
  \begin{equation}\label{eqn:loss} R(t) = \N{\zbs^t(\ybs)-\xbs}_2^2, \quad \widehat{R}(t) = \N{\zbs^t(\ybs)-\widehat{\xbs}}_2^2,\end{equation}
  are called the \emph{true} and the \emph{empirical quadratic loss}, respectively.
  Furthermore, define
  \begin{equation}\label{eqn:topt} t_{\text{opt}} =  \argmin_{t\in[0,1]} R(t), \quad \widehat{t}_{\text{opt}}= \argmin_{t\in[0,1]} \widehat{R}(t).\end{equation}
\end{definition}

In view of Lemma \ref{lem:solutions} and the discussion in Section \ref{sec:enets}, the benefits of recasting $\lambda$ to $[0,1]$ are clear: $R$ and $\hat R$ are both continuous, defined on a bounded interval, and, hence, achieve a minimum. 
Thus, our aim is to minimize $\widehat{R}$ while ensuring $\SN{t_{\text{opt}}-\widehat{t}_{\text{opt}}}$ is small.

Let us discuss some difficulties associated with elastic net minimization which need to be addressed.
On one hand, d a closed form solution of \eqref{eqn:enets_minimization} is available only when  $\Amat^\top\Amat=\Id$ and is otherwise only approximated.
Furthermore, as we will see below, loss functionals $R$ and $\widehat{R}$ are  globally neither differentiable nor convex, but rather only piecewise.
These two issues suggest that their minimizers in general cannot be analyzed in full detail.
Therefore, in the following we split the analysis into a simplified case for $\Amat^\top\Amat = \Id$ where can provide guarantees, and the general case where we provide an efficient algorithm. 
Furthermore, we need to amend the empirical loss function $\widehat{R}$ in the case when $\Amat$ is non-injective.
This is due to the fact that in case of non-linear methods $R(t)$ cannot be reliably estimated outside the kernel of $\Amat$, see \cite{E09} and Figure \ref{fig:loss_functionals}. 
We follow the idea of SURE-based methods \cite{GEE11}, which provide an unbiased estimate of $R(t)$ by projecting the regularized solution onto $\Null^\perp(\Amat)$.
Namely, we define \emph{projected and modified loss functions} $\widehat{R}_{\FP},\widehat{R}_{\FM}\colon[0,1]\rightarrow\bbR$ by
\begin{equation}\label{eqn:loss_projeced} 
  \widehat{R}_{\FP}(t) =\N {\Pmat \zbs^t(\ybs)-\widehat{\xbs}}^2_2, \text{ and } \widehat{R}_{\FM}(t) =\N{\Amat\zbs^t(\ybs) - \widehat{\Pi}\ybs}_2^2, 
  \end{equation}
  where $\Pmat=\Amat^\dagger\Amat$ is the orthogonal projection onto $\Null^\perp(\Amat)$.
Define also 
\begin{equation}\label{eqn:topt_proj} 
   t_\FP= \argmin_{t\in[0,1]} \widehat{R}_\FP(t), \text{ and }  t_\FM= \argmin_{t\in[0,1]} \widehat{R}_\FM(t)
  \end{equation}
Note that to define $\widehat{R}_{\FM}(t)$ we used the fact that $\widehat{\xbs} = \Amat^\dagger \widehat{\Pi}\ybs$, and thus compared to $\widehat{R}_{\FP},$ we avoid the computation of the Moore-Penrose inverse $\Amat^\dagger$, which might be either costly to compute or indeed numerically unstable if $\Amat$ is poorly conditioned.
As we will show in Section \ref{sec:experiments}, and can see in the right-most panel in Figure \ref{fig:loss_functionals}, using $\widehat{R}_{\FP}$ and $\widehat{R}_{\FM}$ instead of $\widehat{R}$ when $\Amat$ is non-injective dramatically improves the performance.
Note that projecting onto $\Null^\perp(\Amat)$ affects makes the loss functional smoother (dampening the gradients).

\begin{figure}[ht!]
\centering
\begin{minipage}[b]{0.32\linewidth}
\centering
\includegraphics[width=\textwidth]{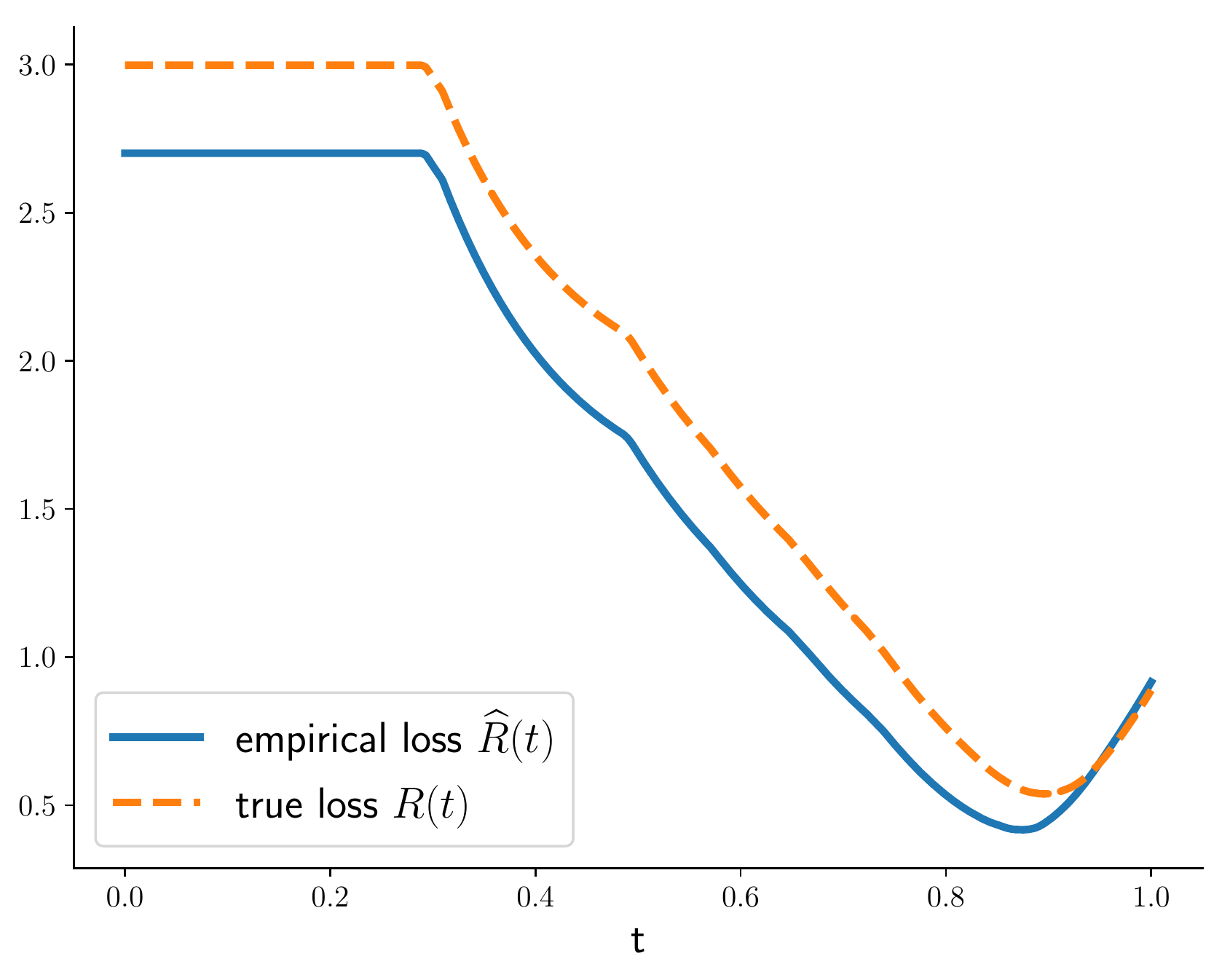}
\end{minipage}
\begin{minipage}[b]{0.32\linewidth}
\centering
\includegraphics[width=\textwidth]{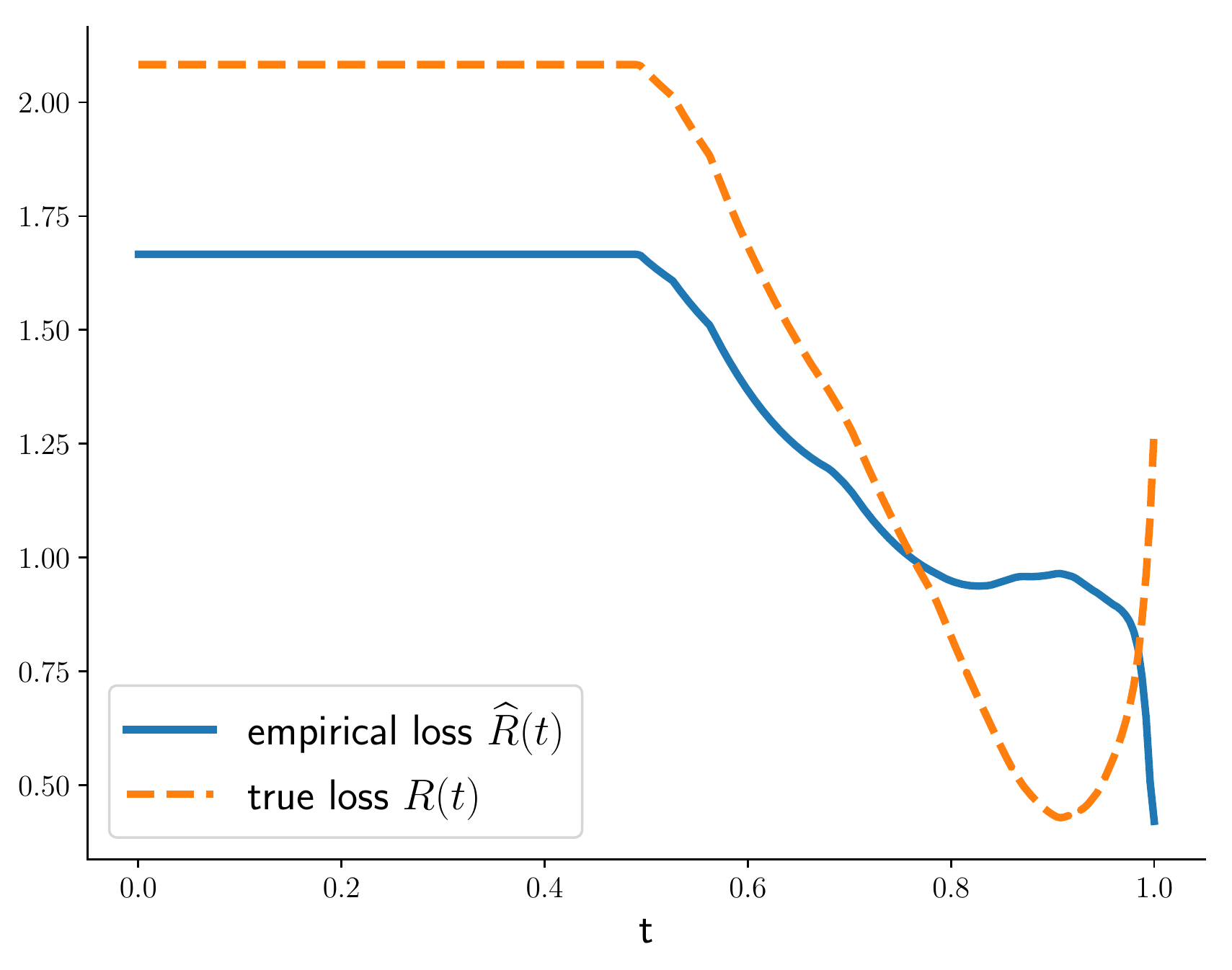}
\end{minipage}
\begin{minipage}[b]{0.32\linewidth}
\centering
\includegraphics[width=\textwidth]{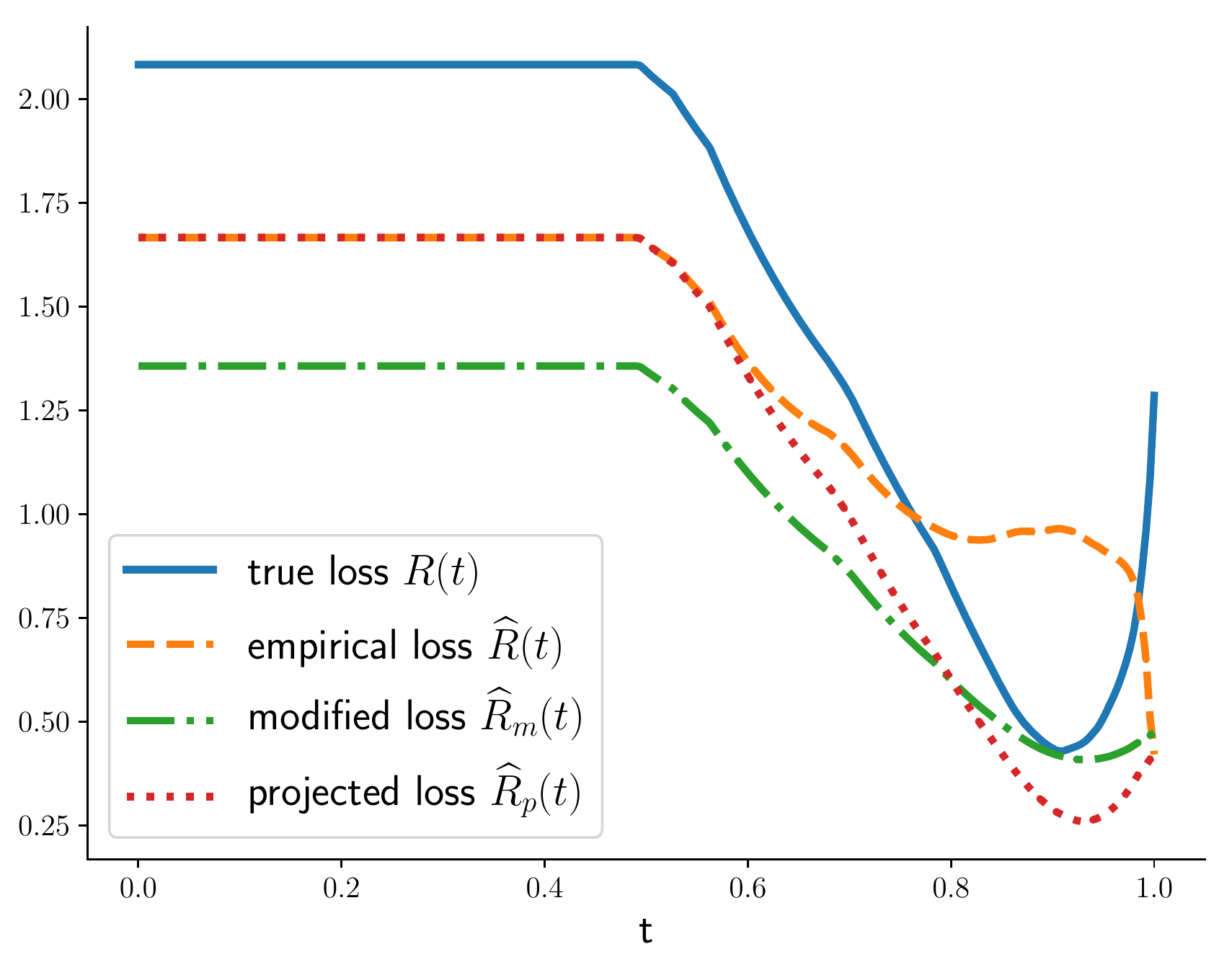}
\end{minipage}

\caption{Empirical and true losses for $m=500$, $d=60$, $h=5$, $N=50$, $\alpha=10^{-3}$, $\sigma =0.08$, and zero mean isotropic Gaussians $\xbs$ and $\wbs$. In the left panel $\Amat$ is injective and $\widehat{t}_{\textrm{opt}}$  is a good proxy for $t_{\textrm{opt}}$. In the middle panel $\rank(\Amat)=40$ and we see that $\widehat{t}_{\textrm{opt}}$ does not approximate $t_{\textrm{opt}}$ well. On the other hand, the right panel shows that in case of a non-injective matrix, $\widehat{R}_\FM$ and $\widehat{R}_\FP$ improve the performance}
\label{fig:loss_functionals}
\end{figure}

\section{Parameter error}\label{sec:error_analysis}

Since the elastic net solution is in general  not available in closed form, a rigorous study of the parameter error is unfeasible in full generality.
Therefore, we restrict our attention to simplified cases, though we emphasize that our approach in practice performs well on significantly broader model assumptions, which we will show in Section \ref{sec:experiments}.
In case of orthogonal design $\Amat^\top\Amat=\Id$ we can, without loss of generality,  assume $\Amat=\Id$ (otherwise redefine $\ybs$ as $\Amat^\top\ybs$).
Let now $\ybs = \xbs+\sigma\wbs$ and assume $\lvert\ysf_1\lvert\geq\ldots\geq\lvert\ysf_m\lvert$.
Plugging \eqref{eqn:naive} into \eqref{eqn:loss} we get
\[ R(t) = \sum_{i=1}^m  \left(\frac{(t(1+2\lvert\ysf_i\lvert)-1)_+}{2(t(1-\alpha)+\alpha)}\Sgn(\ysf_i) - \xsf_i\right)^2.\]
Define {$\edge_i= {{1+2\SN{\ysf_i}}}$, for $i =1,\ldots, m$}. 
Loss function $R(t)$ is continuous on $[0,1]$, and differentiable on intervals 
\begin{align*} 
\CI_0 &= \Big[0, \edge_1^{-1}\Big),\quad \CI_{m} = \Big(\edge_{m}^{-1}, 1\Big], \text{ and } \CI_k = \LRP{\edge_k^{-1}, \edge_{k+1}^{-1}}, \textrm{ for } k=1,\ldots, m-1.
\end{align*}
Considering one interval at a time a direct computation yields that for $k=1,\ldots, m-1$ the minimizer of $R\lvert_{{\CI}_k}$ is
\begin{align}\label{eqn:tstark} t^{\star,k} &= \varphi_k\left(\frac{\sum_{i=1}^k a_i d_i}{\sum_{i=1}^k a_ic_i} \right), \quad
\textrm{ where }\nonumber\\
\varphi_k(t)&=\begin{cases} \edge_k, &\textrm{for } t<\edge_k \\t, & \textrm{for } t\in\CI_k \\ \edge_{k+1}, &\textrm{for } t>\edge_{k+1}
\end{cases}, \textrm{ and } \quad \begin{aligned}
  a_i &= \Sgn(\ysf_i)(1+2\alpha\SN{\ysf_i}), \\ 
  c_i &= \Sgn(\ysf_i)+2\xsf_i(\alpha-1) + 2\ysf_i, \\ 
  d_i &= \Sgn(\ysf_i)+2\alpha\xsf_i.
\end{aligned}
\end{align}
An analogous expression holds for $k=m$, whereas $R(t)$ is constant on $\CI_0$, as argued in \eqref{eq:5}. 
Therefore, the minimizer of $R(t)$ is $t_{\textrm{opt}}=\argmin_{k=0,\ldots,m} R(t^{\ast,k})$. 
The empirical loss function $\widehat{R}(t)$ is also continuous on $[0,1]$ and is piecewise differentiable on the same set of intervals since they depend only on $\ybs$.
Consequently, minimizers $\widehat{t}^{\star,k}$ of $\widehat{R}(t)$ are also of the form \eqref{eqn:tstark}, where we only ought to replace $\xsf_i$ by $\widehat{\xsf}_i$.

Notice that unless further assumptions are made,
minimizers $t_\textrm{opt}$ and $\widehat{t}_\textrm{opt}$ are not given explicitly: we still need to evaluate $R(t)$ and $\widehat{R}(t)$ at $m+1$ locations, and it is not clear that there are no local minima or that the minimizer is unique.
We will now show that in case of bounded sub-Gaussian noise there is indeed only one minimum and that it concentrates near $t=1.0$ for moderate noise levels.
This analysis will also give a theoretical intuition that will drive our algorithm.
Furthermore, we will show that in a simplified case of Bernoulli noise we get explicit bounds on the parameter error.

\subsection{Bounded noise}\label{sec:bounded_noise}
Consider now the case of bounded noise such that there is a gap between the noise and the signal. 
We show that there exists a unique minimizer and there are no local minima.
For simplicity of computation, we let $\alpha=1$, though the results hold for all $\alpha>0$.
Let $\ybs= \xbs+\sigma\wbs$ and assume $\xbs = \LRP{\xsf_1,\ldots,\xsf_h, 0,\ldots,0}^\top$ where $\SN{\xsf_i}>2\sigma\SN{\wsf_j}$ for all $i=1,\ldots,h$ and $j=1,\ldots, m$. 
Without loss of generality, we assume that $\ybs$ is ordered so that
\begin{equation*}\SN{\xsf_i+\sigma\wsf_i}> \SN{\xsf_j+\sigma\wsf_j} \text{ for }1\leq i < j\leq h \text{ and } \SN{\wsf_i}\geq\SN{\wsf_j}\text{ for }1\leq i<j\leq m.
\end{equation*}
The loss functional $R(t) = \N{\zbs^t - \xbs}_2^2$ is thus piecewise differentiable on intervals $\CI_k$, where $\edge_i = 1+2\SN{\xsf_i+\sigma\wsf_i}$ for $i=1,\ldots,h$, and 
$\edge_i = 1+2\sigma\SN{\wsf_i}$ {for } $i=h+1,\ldots,m$.
Also, we have $\edge_i\geq\edge_j$ for $i\leq j$.
We will show that $R(t)$ is decreasing\footnote{We will show that $R(t)$ can be monotonically increasing only for a large enough $t$. Thus, for all $t$ smaller than that value (denoted as $\vartheta_j$), it will be a monotonously decreasing function} for all $t\leq \bsf_{h+1}^{-1}$.
Let thus $t\in\CI_j$ for $j<h$.
The function $R(t)$ is continuously differentiable in $\CI_j$, so it is sufficient to show that $R'(t)$ is positive.
By a direct computation it follows
\begin{equation}\label{eqn:Rderiv} R'(t) \geq 0 \text{ if } t\geq \frac{\sum_{i=1}^j \edge_i\LRP{1+2\Sgn(\ysf_i)\xsf_i }}{\sum_{i=1}^j \edge_i^2}=:\vartheta_j.\end{equation}
It suffices to show $\vartheta_j\geq \edge_{j+1}^{-1}.$
Since $\Sgn(\ysf_i)=\Sgn(\xsf_i)$ for $i\leq j < h$, we have
\begin{align}\label{eqn:outside_the_edge}
\edge_{j+1}\LRP{1+2\SN{\xsf_i}} - \edge_i >4\SN{\xsf_i} \SN{\xsf_{j+1}+\sigma\wsf_{j+1}} >0.
\end{align}
Therefore, $ \edge_i\LRP{1+2\Sgn(\ysf_i)\xsf_i }\geq \edge_{j+1}^{-1}\edge_i^2$, and the claim follows.
Extending the same analysis to  $t\in\CI_h$ what we ought to show is $\edge_{h+1}\LRP{1+2\SN{\xsf_i}} - \edge_i>0$ for  $1\leq i \leq h$.
A direct computation gives
\[\edge_{h+1}\LRP{1+2\SN{\xsf_i}} - \edge_i > 2 \sigma\SN{\wsf_{h+1}}\LRP{1+2\SN{\xsf_i}} >0.\]
Hence, $t_{\text{opt}} > \edge_{h+1}^{-1}$, as desired. 
We will now show that $R(t)$ admits only one minimizer.
Assume there exists $t^\star$ such that $t^\star\in\CI_{j^\star}$ for some $j^\star >h$ and $R'(t^\star)=0$.
This means
\[ t^\star = \frac{\sum_{i=1}^h \edge_i\LRP{1+2\Sgn(\ysf_i)\xsf_i } + \sum_{i=h+1}^{j^\star} \edge_i }{\sum_{i=1}^{j^\star} \edge_i^2}=\vartheta_{j^\star}, \text{ and } \edge_{j^\star}^{-1} < \vartheta_{j^\star}<\edge_{j^\star+1}^{-1}.\]
We proceed by induction showing that $R(t)$ is increasing for all $j>j^\star$.
For $t\in\CI_j$ with $j>h$, it follows
\[ R'(t) \leq 0 \text{ if } t\leq \frac{\sum_{i=1}^h \edge_i\LRP{1+2\Sgn(\ysf_i)\xsf_i } + \sum_{i=h+1}^{j} \edge_i }{\sum_{i=1}^{j} \edge_i^2}=:\vartheta_j.\]
Let us show $\vartheta_j<\edge_j^{-1}$ for  $j=j^\star+1$. 
We have
\begin{align*} 
\vartheta_j &= \frac{\sum_{i=1}^h \edge_i\LRP{1+2\Sgn(\ysf_i)\xsf_i } + \sum_{i=h+1}^{j} \edge_i }{\sum_{i=1}^{j} \edge_i^2}= \frac{\vartheta_{j^\star}{\sum_{i=1}^{j^\star} \edge_i^2} + \edge_j}{\sum_{i=1}^{j} \edge_i^2}\leq \edge_{j}^{-1},
\end{align*}
where we used the fact $\vartheta_{j^\star}\leq \edge_j^{-1} = \edge_{j^\star+1}^{-1}$, and $\edge_j>\edge_{j^\star}$.
The rest of the proof then follows by mathematical induction.
Analogous computation yields the same type of a result for the empirical loss function $\widehat{R}$. 
\begin{lemma}\label{lem:non_bernoulli}
Let the above assumptions hold.
Loss function $R(t)$ is then either monotonically decreasing on the entire interval $[0,1]$, or it is decreasing until some interval $\CI_{j^\star}$, for $j^\star > h+1$ where it achieves a (unique) minimum, and it is monotonically increasing on all the subsequent intervals. 
The same holds for $\widehat{R}$ and all $\alpha>0$.
\end{lemma}

\subsection{Bernoulli noise}\label{sec:bernoulli_noise}
Lemma \ref{lem:non_bernoulli} states that $R(t)$ and $\widehat{R}(t)$ achieve a unique minimum in $[0,1]$, and that they are monotonically decreasing before, and monotonically increasing after this minimum.
Furthermore, the minimizer is bigger than $(1+2\sigma\SN{\wsf_{h+1}})^{-1}$, which means that for moderate noise levels, it will be close to $1$.
The issue is that minimizers $\vartheta_{j^\star}$ and $\widehat{\vartheta}_{\widehat{j}^\star}$ do not need to lie in the same interval, that is $j^\star\neq\widehat{j}^\star$, and thus they cannot be directly compared.
Instead, we consider a simplified model that still encodes the main features of the problem. 
In particular, let
\[ \ybs = \xbs + \sigma\wbs, \textrm{ where } \bbP\LRP{ \wsf_i = \pm 1} = \frac{1}{2},\]
and assume
\( \xbs=(\xsf_1,\ldots,\xsf_h,0,\ldots,0)^\top,\) and \(\SN{\xsf_i}\geq 2\sigma,\) for  \(i=1,\ldots, h\).
As before, without loss of generality we can assume $\lvert\ysf_1\lvert\geq\lvert\ysf_2\lvert\geq\ldots\geq\lvert\ysf_m\lvert$. It then follows 
\( \edge_i=2\SN{\xsf_i\pm\sigma}+1,\) for $1\leq i \leq h$, and $\edge_i=2\sigma  + 1$  otherwise. Moreover, $\edge_j>\edge_{h+1},$ for all $j=1,\ldots, h$.
In the following, we will for the sake of simplicity consider the case $\alpha=1$. 
The details regarding the general case, $\alpha \neq 1$, are in the Appendix.
First, as in Section \ref{sec:bounded_noise} we know that $t_\textrm{opt}\geq\edge^{-1}_{h+1}.$
We can now explicitly compute the minimizer of $R(t)$.

\begin{lemma}\label{lem:bernoulli_minimizer}
True loss functional $R(t)$ is minimized for $t_{\textrm{opt}} = \min\{t^\star,1\}\in [\edge_{h+1}^{-1},1]$, where
\begin{equation}\label{eqn:bernoulli_tstar} t^\star = \frac{\sum_{i=1}^h\edge_i\LRP{1+2\Sgn\LRP{\ysf_i}} + \edge_{h+1}(m-h)}{\sum_{i=1}^h \edge_i^2 + (m-h)\edge_{h+1}^2}.\end{equation}
\end{lemma}
\begin{proof}
Considering \eqref{eqn:Rderiv} for $t\in \LRP{\edge_{h+1}^{-1},1}$ we get
\begin{align}\label{eqn:der_at_one} 
2R'(t) &= \LRP{\sum_{i=1}^h \edge_i^2 + (m-h)\edge_{h+1}^2}t -\sum_{i=1}^h\edge_i\LRP{1+2\Sgn\LRP{\ysf_i}} - \edge_{h+1}(m-h).
\end{align}
The root of \eqref{eqn:der_at_one} is exactly \eqref{eqn:bernoulli_tstar}.
Arguing as in \eqref{eqn:outside_the_edge} we have $t^\star>\frac{1}{\edge_{h+1}}$.
Restricting to $[0,1]$ the claim follows.
\end{proof}
\begin{remark}
The minimizer given by Lemma \ref{lem:bernoulli_minimizer} will be in $[0,1]$ provided $\sum_{i=1}^h \SN{\ysf_i}\leq (m-h)\sigma$ and $h\leq m/2$.
\end{remark}
For the empirical loss function it is in general not true that $\widehat{\xsf}_i=0$ for $i>h$, nor is $\frac{\ysf_i-\widehat{\xsf}_i}{\sigma}$ a Bernoulli random variable.
However, $\ybs$ and $\edge_i$'s remain the same, and an entirely analogous computation gives 
\begin{equation}\label{eqn:t-t} \widehat{t}^\star = t^\star +
  \frac{\sum_{i=1}^m \edge_i \Sgn\LRP{\ysf_i}(\xsf_i-\widehat{\xsf}_i)
  }{\sum_{\edge_i>1/t^{\star}}\edge_i^2}.
\end{equation}
We can now bound the approximation error for the optimal regularization parameter.
\begin{theorem}\label{prop:error_estimate}
Assume that $t_{\textrm{opt}}<1$ and $\sigma\frac{h}{m}<1$. Given $u>0$,   with probability of at least $1-2\exp(-u)$ we have
  \begin{equation}\label{eqn:bernoulli_unbounded} \SN{t_{\textrm{opt}} - \widehat{t}_{\textrm{opt}}} \leq
          \frac{\lambda_1}{\lambda_h}
          \LRP{\sqrt{\frac{h+\sigma^2 m+u }{N}} +
            \frac{h+\sigma^2 m+u}{N}}  + \sigma\sqrt{\frac{h}{m}},\end{equation}
provided $N\gtrsim \LRP{h+\sigma^2m +u}$. 
{Assume now $\ybs$ is bounded.  With
probability greater than $1 - 3 \exp\LRP{-u}$ we then have
\begin{equation}\label{eqn:bernoulli_bounded} \SN{t_{\textrm{opt}} - \widehat{t}_{\textrm{opt}}} \leq
          \frac{\lambda_1}{\lambda_h}
          \LRP{\sqrt{\frac{\log(h+m)+u}{N}} +
            \frac{\log(h+m)+u}{N} } + \sigma\sqrt{\frac{h}{m}} ,\end{equation}
provided $N\gtrsim \LRP{\log{2m} + u}$. }
\end{theorem}
\begin{proof}
By assumption and Lemma \ref{lem:bernoulli_minimizer} $t_{\textrm{opt}} =t^\star$. We can now rewrite \eqref{eqn:bernoulli_tstar} and \eqref{eqn:t-t} as
  \[ t^\star = \frac{\N{\bbs}_1 + 2\EUSP{
  (\Sgn{\ybs}\cdot\bbs)}{\xbs} } {\N{\bbs}_2^2},\quad \widehat
  t^\star = \frac{\N{\bbs}_1 + 2\EUSP{
  (\Sgn{\ybs}\cdot\bbs)}{\widehat\xbs} } {\N{\bbs}_2^2}.\]
  Therefore, using \eqref{eqn:t-t} we have
\begin{align*} t^\star - \widehat{t}^\star =\frac{ \EUSP{\vbs}{\xbs-\widehat\xbs}}{\N{\bbs}_2^2},
\end{align*}
with $\vbs$ defined by $\vsf_j=2\Sgn(\ysf_j)\edge_j$ for $j=1,\ldots, h$, and $\vsf_j=\Sgn(\wsf_j)$ for $j>h$.
Thus, $\N{\vbs}_2\leq\N{\bbs}_2$ and we have
\begin{align*}
\SN{t^\star - \widehat{t}^\star} &\leq 2 \frac{\N{\xbs-\widehat\xbs}_2} {\N{\bbs}_2} \leq 2\LRP{\frac{\N{\ybs}_2}{\N{\bbs}_2}\N{\Pi-\widehat\Pi}_2
                                             +\sigma\frac{\N{\Pi\wbs}_2} {\N{\bbs}_2}}. 
\end{align*}
Using $\N{\bbs}_2^2 = m + 4\N{\ybs}_1 + 4\N{\ybs}_2^2$ and $\N{\Pi\wbs}_2={h}$, and provided $N\gtrsim \LRP{h+u+\sigma^2m }$, by Lemma \ref{lem:proj_error} we have
\[
\SN{t^\star - \widehat{t}^\star} \leq C \frac{\lambda_1}{\lambda_h}
          \LRP{\sqrt{\frac{h+\sigma^2 m+u}{N}} +
            \frac{h+\sigma^2 m+u}{N}}  + \sigma\sqrt{\frac{h}{m}},\] 
with probability greater than $1-2\exp(-u)$, and $C>0$ is a constant. 
Since $\SN{t^\star-\widehat{t}^\star}\geq \SN{t^\star - \widehat{t}_{\textrm{opt}}} $ the claim follows. 
The proof for a bounded  $\ybs$ is entirely analogous.  
\end{proof}

\begin{remark}
Theorem \ref{prop:error_estimate} is valid not only for $\alpha=1$ but for all $\alpha$.
Namely, in Appendix \ref{app:alphaneq1} we show
\[\SN{t_{\textrm{opt}} - \widehat{t}_{\textrm{opt}}} \lesssim
          \frac{\lambda_1}{\lambda_h}
          \LRP{\sqrt{\frac{h+u+\sigma^2 m }{N}} +
            \frac{h+u+\sigma^2 m}{N}}  + \sigma\sqrt{\frac{h}{m}},\] 
holds for any $\alpha>0$.
\end{remark}

\section{OptEN algorithm}\label{sec:general_matrices}

Driven by insights in Section \ref{sec:error_analysis}, we are ready to present an efficient heuristic algorithm for learning the {\it Opt}imal regularization parameter for the {\it E}lastic {\it N}et (OptEN). 
The algorithm is based on the minimization of a given loss function ($\widehat{R}$, $\widehat{R}_\FM$ or $\widehat{R}_\FP$).
In Section \ref{sec:error_analysis} we showed that in a simplified, yet instructive, setting that the optimal parameter tends to be in the vicinity of $t=1$, depending on the noise level and the signal-to-noise gap.
This is supported by experimental evidence in more general situations such as for non-injective $\Amat$, as we will see in Section \ref{sec:experiments}.
Moreover, the loss function is monotonically decreasing as we get away from $t=1$.
These observations drive our algorithm which assumes that the minimizer lies in a valley not too far from $t=1,$ see Figure \ref{fig:loss_functionals}.
Therefore, we will perform a \emph{line search} on the graph of a given loss function, starting from $t=1$. 

Line search methods follow iterations $t_{k+1} = t_k + s_k \pbs_k$, where $\pbs_k$ is the \emph{search direction} and $s_k$ the \emph{step size}:
\begin{itemize} 
\item {\bf Search direction.}
We select $p_k$ by estimating $\widehat{R}'(t_k)$ with central differences,
\( \widehat{R}'(t) \sim \Delta_\epsilon R(t) := \frac{R(t+\epsilon)-R(t-\epsilon)}{2\epsilon}\)
where $\epsilon>0$.
For $t=1$ we instead use  $\tilde\Delta_\epsilon R(1) := \frac{R(1)-R(1-\epsilon)}{\epsilon}$. Then set $p_k = - \Delta_\epsilon R(t).$
\item {\bf Step size.} We estimate $s_k$ with the backtracking line search (consult \cite{A06} for an overview of line search methods). 
\end{itemize}
Our approach is presented in Algorithm \ref{alg:algo}, while an extensive numerical study is provided in the next section.
\begin{algorithm}[ht!]
\begin{algorithmic}
  \small{
   \STATE {\bfseries Input:} $\ybs_1,\ldots,\ybs_N,\ybs\in\bbR^m,\,\Amat\in\bbR^{m\times d}$;
\vspace{5pt}
\STATE Compute $\widehat{\xbs}$ according to \eqref{eqn:xhat}. 
\STATE Set a loss {function} $\rightarrow\,\,\widehat{R} \textrm{ or } \widehat{R}_\FP \textrm{ or } \widehat{R}_\FM$. In the rest of the algorithm we will refer to it as $R$;
\STATE Set $\epsilon>0$, $\texttt{tol}>0$, $\texttt{tol2}>0$, $0<\alpha<1$, and $c_1,\beta,\gamma>0$;
\STATE Set $k\leftarrow0$, $t_0 = 1$
\STATE Compute $r_1=R(1)$, $\tilde{r}_1=R(1-\epsilon)$, $\psf_0=(r_1-\tilde r_1)/\epsilon$, and $r_2 = \varphi(\gamma_0)$;
   \REPEAT
    \STATE $\tilde{t} = t_k + \alpha \psf_k$;
    \STATE $\varphi_0=r_1$, $\varphi_0'=-\psf_k^2$, $\varphi_1=R(\tilde t)$;
    \IF{$\varphi_1-\varphi_0 < c_1 \varphi_0'\alpha$}
      \STATE $s_k=\alpha$;
    \ELSE 
      \STATE $s_k = -\frac{1}{2}\frac{\varphi_0'\alpha^2}{\varphi_1- \varphi_0 -\varphi_0'\alpha}$;
    \ENDIF
    \IF{$|s_k| < \texttt{tol2}$ or $|s_{k-1}/s_k| > \gamma$}
      \STATE $s_k = s_{k-1} \cdot \beta$;
    \ENDIF
    \STATE Set $t_{k+1}=t_k + s_k \psf_k$;
    \STATE Compute $r_1=R(t_{k+1})$, $\psf_{k+1} = (R(t_{k+1}+\epsilon)-R(t_{k+1}-\epsilon)/(2\epsilon)$;
    \STATE $k\leftarrow k+1$;
   \UNTIL{$|p_k| < \texttt{tol}$ or $k<\texttt{max\_iter}$};
   }
  \STATE \textbf{Output}: Approximate regularization parameter $\hat{t}:=t_k$.
\end{algorithmic}
   \caption{{\bf OptEN algorithm} for approximating the optimal elastic net regularization parameter using backtracking line search}
   \label{alg:algo}
\end{algorithm}

\section{Experimental results}\label{sec:experiments}

We now study the performance of our approach and show its adaptivity to different scenarios by conducting experiments on synthetic and imaging data. 
In the first set of experiments we perform a thorough comparison of our method with state-of-the-art parameter selection rules by exploring their behavior with respect to noise level and other notions. 
The second set of experiments deals with image denoising where we use wavelet-based thresholding with elastic nets. 
We consider two data-sets: natural images and a real-world brain MRI data.
Note that we do not aim to compare our method with state-of-the art denoising methods, but rather only with state-of-the-art methods regarding the selection of the regularization parameter for the elastic net.
We start with a discussion of methods that can be used for the automatic detection of the sparsity level $h$ and show that when a sufficient amount of training points is given, we can reliably estimate $h$. 

\subsection{Estimating the sparsity level}\label{sec:estimating_sparsity}

In real applications the sparsity level of a vector is either not available or is only an approximate notion, \emph{i.e.} the desired vector is sparse only when we threshold its entries.
Such regimes require $h$ to be estimated, which in our case means looking at the spectrum of the corresponding covariance matrix.
This question belongs to the class of low-rank matrix recovery problems since what we are trying to recover is the geometry (\emph{i.e.} projection onto the range) of the noiseless, lower rank matrix $\Sigma(\Amat\xbs)$, using only the covariance matrix of noisy observations $\widehat\Sigma(\ybs)$, which is of full rank. 
Thus,  estimating $h$ boils down to thresholding singular values of the empirical covariance matrix according to some spectral criteria that exploits the underlying structure.

\begin{figure}[ht!]
\centering
\includegraphics[width=\textwidth]{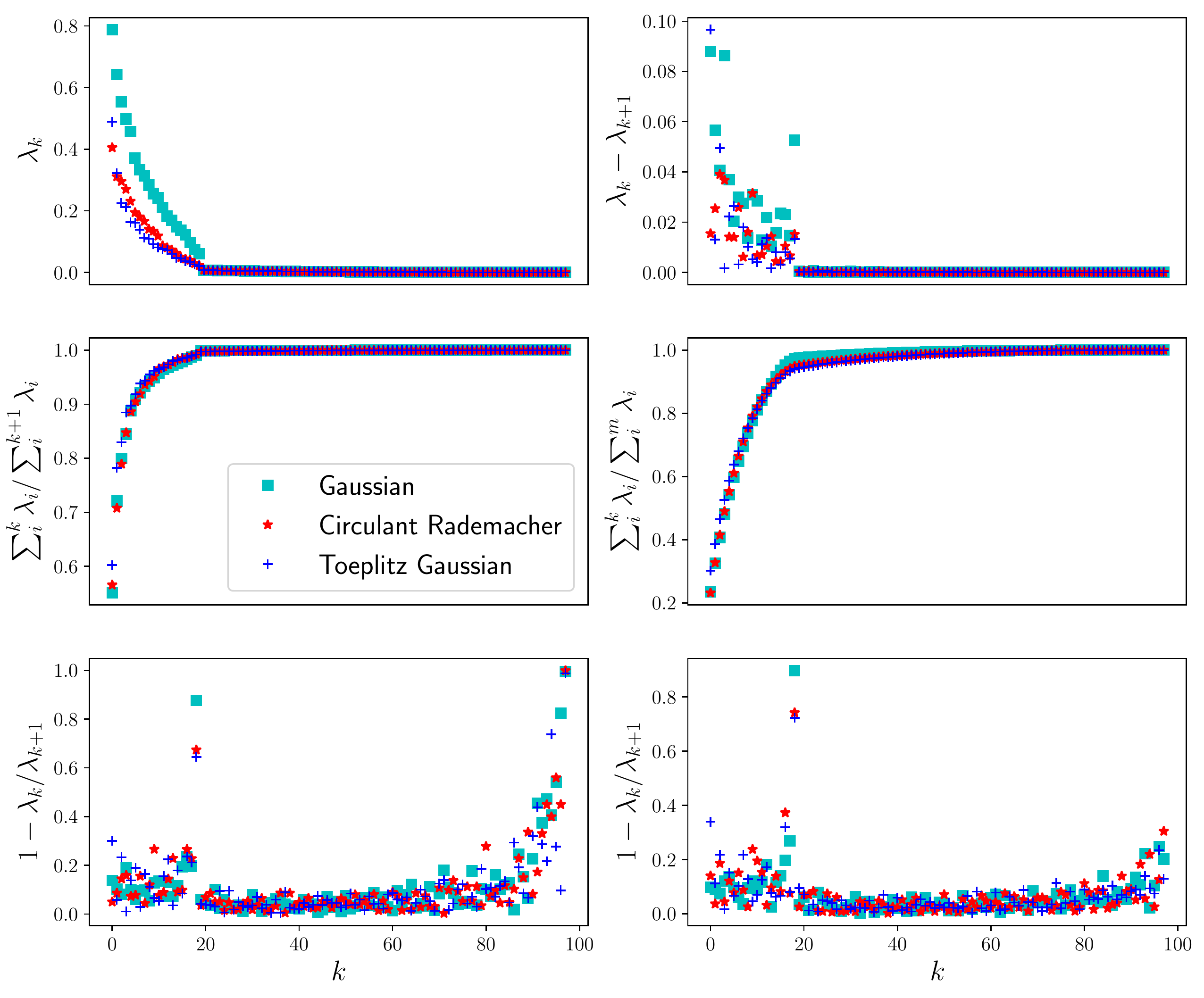}
\caption{Plot in the upper left corner shows the spectrum of three different types of matrices. The remaining plots (not including the one in the bottom right corner) consider different notions of the spectral gap for $N=100$. The plot in the bottom right considers the last criteria, $1-\frac{\lambda_k}{\lambda_{k+1}}$ but for $N=150$ samples, showing that the behavior for large $k$ changes dramatically, compared to the plot in the bottom right for $N=100$}
\label{fig:spectral_gap}
\end{figure}

For a positive definite matrix  with singular values $\lambda_1\geq \ldots\geq \lambda_m\geq 0$,  commonly used spectral criteria are (a) the spectral gap $\argmax_k \SN{\lambda_{k} - \lambda_{k+1}}$; (b) the relative gap $\argmax_{k}\LRP{1 - \frac{\lambda_k}{\lambda_{k+1}}}$; (c) the cumulative spectral energy $\sum_{i=1}^k \lambda_i/\sum_{i=1}^m \lambda_i$, and (d) the relative cumulative spectral energy $\sum_{i=1}^k \lambda_i/\sum_{i=1}^{k+1} \lambda_i$.
For the latter two criteria one sets a threshold, say $0.95$, and selects $\tilde h$ as the first $k$ for which the corresponding spectral energy reaches that threshold.

We study the behavior of these four criteria on three different types of forward matrices $\Amat$: random Gaussian, random circulant Rademacher, and random Toeplitz Gaussian matrices.
These matrices were chosen because they have different spectral behavior and commonly appear in inverse problems. 
In each case $\Amat$ is a $100\times 100$ real matrix, normalized so that $\N{\Amat}=1$, and we take $N=100$ samples $\xbs_i$, sampled according {to \ref{data:D2}} for $h=20$, $\wbs_i\sim\CN(\boldsymbol{0},\Imat_{100})$ and $\sigma=0.3$. We compute $\widehat\Sigma(\ybs) = \frac{1}{N} \sum_{i=1}^N \ybs_i\otimes\ybs_i$ for $\ybs_i=\Amat\xbs_i+\sigma\wbs_i$.

In Figure \ref{fig:spectral_gap}, we show the application of the aforementioned spectral criteria to $\widehat{\Sigma}(\ybs)$.
It is clear from the results that all four methods would fail if used without taking further information into account. 
For example, the spectral gap criterion (in the upper right panel) would dictate {the selection of $h=1$}, but a more careful look at the plot suggests that the behavior of the spectral gap changes dramatically around $h=20$, which corresponds to the true $h$.
Such \emph{ad hoc} solutions are sensible and often improve the performance but can be hard to quantify, especially on real data.

The last spectral criteria, $1-{\lambda_k}/{\lambda_{k+1}}$, is perhaps the most promising, but is also subject to  demands on $N$, as shown in the bottom row of Figure \ref{fig:spectral_gap}. 
Namely, if $N$ is not large enough then $1-{\lambda_k}/{\lambda_{k+1}}$ has a heavy tail on the spectrum of $\widehat{\Sigma}(\ybs)$ and would thus suggest a large $h$, as in the bottom left corner of the Figure.
Instead, in Section \ref{sec:synth_experiments} we look for the relative gap within the first $m/2$ singular vectors.
We note that in the case of the first three spectral criteria the situation does not change as the number of samples increases.
On the other hand, for the last criterion it does: heavy tails flatten back to zero for all three choices of random matrices, see the plot in the bottom right corner of Figure \ref{fig:spectral_gap}.

In Section \ref{sec:wavelet_denoising} we will apply our algorithm to wavelet denoising where there is no natural choice of $h$ since wavelet coefficients of images are not truly sparse.
We will instead consider two scenarios; when we are given an oracle $h$ (\emph{i.e.} the $h$ giving the highest PSNR), and when $h$ has to be estimated from data.

\subsection{Synthetic examples}\label{sec:synth_experiments}

\paragraph{Experimental setting.} 
We consider the inverse problem of the type $\ybs=\Amat\xbs + \sigma \wbs$, where the data are generated according to $\xbs = \xibf + {\vbs}$ , {where
\begin{enumerate}[label=(\textsf{D}\arabic*)]
  \item\label{data:D1} $\Amat\in\bbR^{m\times d}$ is a random Gaussian matrix such that $\N{\Amat}_2=1$,  
  \item\label{data:D2} $\xi_i\sim\CN(0, 1)$, and $\vsf_i= 4\Sgn(\xi_i)$ for $1\leq i \leq h$; $\xi_i=\vsf_i=0$ otherwise,
  \item\label{data:D3} $\wbs\sim\CN(\boldsymbol{0},\Id_m)$.
\end{enumerate}
}
{
The rationale behind distributional choices in \ref{data:D2} is twofold. 
First, having $\vbs$ be a non-constant vector ensures that $\CV=\range (\cov{\xbs})$ is truly and fully $h$-dimensional (\emph{i.e.} if $\vbs$ were constant there would be one dominant singular vector). 
Second, forcing $\SN{\vsf_i}= 4$ for $i=1,\ldots, h$ ensures  that the data are not concentrated around the origin (which happens \emph{e.g.} if $\vsf_i=0$ for all $i$ in \ref{data:D2}) and that there is a gap between the original signal and the noise.
The gap can be measured by $\operatorname{SparseSNR}:=\frac{\max_{1\leq i\leq m} \SN{\sigma \wsf_i}}{\min_{1\leq i\leq h} \SN{\xsf_i}}$. }
We add that the conclusions and the results of this section, and of Section \ref{sec:estimating_sparsity}, stay the same in case of the more usual distribution assumptions, \emph{i.e.} for $\xbs\sim\CN(\boldsymbol{0}, \Id_h)$ where $\Id_h$ is a diagonal matrix with exactly $h$ entries set to $1$ and the rest to $0$, and noise levels as in Section \ref{sec:wavelet_denoising}.

\paragraph{Comparison.}  
We compare our algorithm with the following parameter selection methods:
the discrepancy principle \cite{morozov2012}, monotone error rule \cite{TH99}, quasi optimality \cite{TG65}, L curve method \cite{hansen1992}, (Monte-Carlo) balancing principle \cite{lepskii1991problem} and its elastic net counterpart \cite{DDR09}, (Monte-Carlo) generalized cross-validation \cite{GEE11} and nonlinear cross validation \cite{F05}. In the remainder of this paper we refer to other methods by their acronyms, and to our method as \emph{OptEN}. 
The first five methods are commonly used in inverse problems (a detailed account and an experimental study can be found in \cite{zbMATH05929140}), 
whereas Monte-Carlo and nonlinear cross-validation are adaptations of generalized cross-validation for non-linear regularization methods. 

Before presenting the results, we provide a concise description of considered methods. 
Most of the methods require some additional information about the problem, predominantly the noise level $\sigma$, to be either known or estimated, which affects their performance.
We provide the true noise level whenever a given method requires it and furthermore, we perform judicious testing and tuning of all other quantities, taking into account recommendations from relevant literature.
We consider a regularization parameter sequence $t_n = \frac{1}{1+\mu_0 q^n}$, 
where $n\in\{0,1,\ldots, N_{\max}\}$, and $\mu_0>0$, $q>0$ and $N_{\max}\in\bbN$ are preselected\footnote{This is an adaptation of the parameter sequence from \cite{zbMATH05929140} that reflects our reparametrization from $\lambda$ to $t$, as in \eqref{eqn:enets_minimization}}. 
For each $n$ we denote the corresponding elastic nets solution as $\zbs_n:=\zbs^{t_n}$. 

\paragraph{Discrepancy Principle [DP]\\}
Discrepancy principle is one of the oldest parameter choice rules which selects a solution so that the norm of the residual is at the noise level. 
Thus, the regularization parameter is chosen by the first $n\in\bbN$ such that
\begin{equation}\label{eqn:dp} \N{\Amat \zbs_n - \ybs}_2 \leq \tau \sigma \sqrt{m},\end{equation}
where we fix $\tau=1$.

\paragraph{Monotone Error Rule [ME]\\}

This rule is based on the observation that the monotone decrease of the error $\N{\zbs_n - \xbs}_2$ can only be guaranteed for large values of the regularization parameter. Therefore, the best parameter $t_{n^*}$ is chosen as the first $t$-value for which one can ensure that the error is monotonically decreasing. The parameter is then chosen by the smallest $n$ such that
\begin{equation}\label{eqn:mer} 
\frac{\EUSP{\Amat\zbs_n-\ybs}{\Amat^{-\top}\LRP{\zbs_n - \zbs_{n+1}}} }{\|  \Amat^{-\top}\LRP{\zbs_n - \zbs_{n+1}} \|_2} \leq \tau \sigma \sqrt{m}.
\end{equation}
We fix $\tau=1$ for our experiments.  
The left hand side of \eqref{eqn:mer} is replaced with \eqref{eqn:dp} whenever the denominator is $0$.

\paragraph{Quasi-Optimality Criterion [QO]\\}

Quasi-optimality is a parameter rule that does not need the noise level, and thus has enjoyed reasonable success in practice, especially for Tikhonov regularization and truncated singular value decomposition.
The regularization parameter is chosen according to
\begin{equation}\label{eqn:qoc} n_\star = \argmin_{n\leq N_{\max}} \N{\zbs_n-\zbs_{n+1}}_2.\end{equation}

\paragraph{L-curve method [LC] \\}
The criterion is based on the fact that  the $\log-\log$  plot of $(\N{\Amat\zbs_n-\ybs}_2,\N{\zbs_n}_2)$ often has a distinct L-shape.
As the points on the vertical part correspond to under-smoothed solutions, and those on the horizontal part correspond to over-smoothed solutions, the optimal parameter is chosen at the elbow of that \emph{L-curve}. There exist several versions of the method; here we use the following criterion
\begin{equation}\label{eqn:lc}n_\ast = \argmin_{n\leq N_{\max}} \{\N{\Amat\zbs_n-\ybs}_2\N{\zbs_n}_2\}.\end{equation}

\paragraph{(Monte-Carlo) Balancing Principle [BP]\\}
The principle aims to balance two error contributions, approximation and sampling errors, which have an opposite behavior with respect to the tuning parameter. More precisely, we select the parameter by 
\[
n^* = \argmin_n \{ t_n |   \N{\zbs_n-\zbs_k}_2 \leq 4 \kappa \sigma \rho(k), k = n,\ldots, N_{\max}  \},
\]
where $\kappa>0$ is a tuning parameter. 
More computationally friendly, yet equally accurate, versions of the balancing principle are also available  \cite{zbMATH05929140}. As our main focus on the accuracy of the parameter choice, we will use the original and more computationally heavy version of the balancing principle. 

The value of $\sigma \rho(k)$ is in general unknown but it can be estimated in case of white noise. Following \cite{zbMATH05929140}, we calculate
\( \rho(k)^2\approx \operatorname{mean} \{  \N{\Amat_n^{-1}\xibf_i}_2^2\},\)
where $\xibf_i\sim\CN(\boldsymbol{0},\Id_m)$, $1\leq i \leq L$ (we use $L=4$), and $\Amat_n^{-1}$ is the map that assigns $\ybs$ to $\zbs_n.$

\paragraph{Elastic Nets Balancing Principle [ENBP]\\}
In \cite{DDR09} the authors propose a reformulation of the balancing principle for elastic net, 
\[
n^* = \argmin_n \{ t_n |   \N{\zbs_k-\zbs_{k+1}}_2 \leq \frac{4 C}{\sqrt{d}\alpha \mu_0 q^{k+1} }, k =N_{\max} -1 ,\ldots,n     \},
\]
The method stops the first time two solutions are sufficiently far apart. The constant $C$ needs to be selected, and in our experience this task requires a delicate touch.

\paragraph{(Monte-Carlo) Generalized Cross-Validation [GCV]\\}

The rule stems from the ordinary cross-validation, which considers all the \emph{leave-one-out} regularized solutions and chooses the parameter that minimizes the average of the squared prediction errors. 
Specifically, GCV selects $n$ according to 
\begin{equation}\label{eqn:gcv} n_\ast = \argmin_{n\leq N_{\max}} \frac{m^{-1}\N{\Amat\zbs_n - \ybs}_2^2}{\LRP{m^{-1} \tr(\Id_m - \Amat\Amat_n^{-1})}^2},\end{equation}
where $\Amat_n^{-1}$ is the map such that $\zbs_n = \Amat_n^{-1} (\ybs)$.
In the case of elastic nets the map $\Amat_n^{-1}$ is not linear and, thus, we cannot assign a meaning to its trace.
Instead, we follow the ideas of \cite{GEE11} and estimate the trace stochastically using only one data sample.

\paragraph{Nonlinear Generalized Cross-Validation [NGCV]\\}
In \cite{F05} the authors reconfigure GCV for non-linear shrinkage methods, and $n$ is selected according to
\[ n_\ast = \argmin_{n\leq N_{\max}} \frac{\N{\Amat\zbs_n - \ybs}_2^2}{m^{-1}\LRP{1- ds/m}^2},\]
where $s = \frac{\N{\zbs_n}_\gamma}{\N{\zbs^\dagger}_\gamma}$ with $\N{\cdot}_\gamma := \N{\cdot}_1+\alpha\N{\cdot}^2_2.$

\begin{table}[th]
\centering
\ra{1.1} 
\begin{tabular}{cccccc}Method & $\frac{|t_{opt} - \widehat{t}|}{t_{opt}}$& $\frac{\|\xbs - \zbs^{\widehat{t}}\|}{\|\xbs\|}$& $\text{FDP}(\widehat{t})$& $\text{TPP}(\widehat{t})$& $\genfrac{}{}{0pt}{}{\textrm{computational}}{\textrm{time [s]}}$\\
\hline\hline
$x_{t_{opt}}$ & 0 & 0.0984 & 0.1583 & 1.000 & 0\\
\hline
$\widehat\xbs$ & N/A & 0.1004 & \textbf{0.000} & \textbf{1.000} & N/A\\
OptEN & \textbf{0.0254} & \textbf{0.0994} & 0.160 & \textbf{1.000} & 3.716\\
DP & 0.1196 & 0.1118 & 0.087 & \textbf{1.000} & 0.7421\\
ME & 0.2493 & 0.1376 & 0.010 & \textbf{1.000} & \textbf{0.1757}\\
QO & 0.4543 & 0.1843 & 0.716 & \textbf{1.000} & 7.193\\
LC & 0.1414 & 0.1174 & 0.188 & \textbf{1.000} & 1.564\\
BP & 0.0877 & 0.1057 & 0.095 & \textbf{1.000} & 36.77\\
ENBP & 0.2728 & 0.1450 & \textbf{0.007} & \textbf{1.000} & 4.355\\
GCV & 0.4548 & 0.1844 & 0.716 & \textbf{1.000} & 15.91\\
NGCV & 0.3597 & 0.1577 & 0.638 & \textbf{1.000} & 7.306
\end{tabular} 
\caption{Comparison of errors for regularization parameter selection methods, with an injective matrix $\Amat\in\bbR^{500\times 100}$ and $h=10$, $\alpha=10^{-3}$, $\sigma=0.3$. The values are averages over $100$ independent runs}
\label{tab:rando_label_3119} 
\end{table}

\paragraph{Comparison/Error.} 
For each method we compute: the (normalized) error in approximating the optimal regularization parameter, the (normalized) error, false discovery proportion (FDP), true positive proportion (TPP), and the computational time.
TPP and FDP are measures that quantify the notions of true and false discovery of relevant features in sparsity based regression tasks \cite{SBC16}.
FDP is the ratio between false discoveries and the total number of discoveries,
\[ \operatorname{FDP}(t) = \frac{\#\left[{j:\, \zsf^t_j \neq 0 \textrm{ and } \xsf_j = 0}\right]}{\max\LRP{\#\left[{j:\, \zsf^t_j \neq 0}\right], 1}}.\]
TPP on the other hand is the ratio between  true (\emph{i.e.} correct) discoveries in the reconstruction and true discoveries in the original signal,
\[ \operatorname{TPP}(t) =\frac{\#\left[{j\in\{1,\ldots, h\}:\, \zsf^t_j \neq 0 \textrm{ and } \xsf_j \neq 0}\right]}{h} .\]
Thus, to recover the structure of the original sparse data we want FDP close to $0$ and TPP close to $1$.
It is known that there is often an explicit (and sometimes even quantifiable) trade-off between FDP and TPP, in the sense that the support overestimation is an (undesirable) side-effect of full support recovery. 
In other words, a consequence of true support discovery is often a non-trivial false support discovery \cite{SBC16}.
When computing FDP and TPP we will rather than demand for an entry to be exactly zero, instead threshold the values (with $0.5$ being the threshold).

\paragraph{Testing setup.} 
To compute the true optimal parameter $t_{opt},$ we run a dense grid search on $[0,1]$ using the true expected loss $\N{\zbs^t-\xbs}_{2}^2$. 
As suggested in \cite{zbMATH05929140} we use $\tau=1$ and provide the true noise level $\sigma$ for discrepancy principle and monotone error rule; balancing principle uses $\kappa = 1/4$ and true $\sigma$; elastic net balancing principle uses $C=1/2500$. 
The parameter grid for DP, ME, BP, QO, LC, BP, GCV, and NGCV is defined by $\mu_0=1$, $q=0.95$, and $N_{\max}= 100$ (thus, $t_0=0.5$ and $t_{N_{\max}}= 0.9941143171)$, whereas for  
ENBP, we use $t_0=0.05$, $q=1.05$, and $N_{\max}= 100$.
The tests are conducted for $m\in\{500, 900\},\,d\in\{100,200\}$ and $h\in\{10,20,30\}$, where all combinations of $\alpha\in\{10^{-5}, 10^{-3},10^{-2},10^{-1}\}$ and $\sigma\in\{0.05, 0.1, 0.2, 0.3\}$ are considered.
To compute the {empirical estimator} $\widehat{\xbs}$, we generate $N = 50$ independent random samples of the training data ($\xbs, \ybs$).

Results in Table \ref{tab:rando_label_3119} are averaged over $100$ independent runs for $\alpha=10^{-3}$, $m=500$, $d=100$, $h=10$, where $\CV=\Span{\ebs_1,\ldots,\ebs_h}$, and $\sigma =0.3$, which corresponds to $\operatorname{SparseSNR}\approx 0.17$. 
The first row in the table, $x_{t_{opt}}$, describes the elastic net minimizer for which the \emph{true} optimal regularization parameter is provided. 

\paragraph{Discussion.}
OptEN always returns the value which is the closest to the optimal regularization parameter, and its results are in general comparable to the ones provided by the minimizer with the optimal parameter. 
However, one can observe that other methods, \emph{e.g.}, discrepancy principle, provide a better balance between FDP and TPP (returning solutions that are more sparse), though at a cost of a larger approximation error. 
Balancing principle  also provides very good results, but it is slow unless an effort is made to improve its computational time.
Moreover, we observed that the performance of all methods that require the noise level $\sigma$ to be known deteriorates if we do not provide the exact value of the noise level, but only its rough estimate.

The overall results are mostly consistent over all experimental scenarios we looked at, with a couple of exceptions.
As expected, FDP and estimation errors deteriorate  not only for larger $\sigma$ but also for larger $\alpha$, though the ranking of the methods and the patterns of behavior remain the same.
This is due to the fact that as $\alpha$ increases elastic nets sacrifice sparsity for smoothness. 
The {empirical estimator} $\widehat\xbs$ is a very accurate estimator of the original signal, and it sometimes outperforms even the elastic net solution that uses the optimal parameter. However, as has been observed in \cite{DFN16} for Tikhonov regularization and confirmed in Figure \ref{fig:noise_levels}, the performance of the {empirical estimator} worsens in the small noise regime.

\paragraph{Comparison with {empirical estimator}: effects of $\sigma$ and $N_{\text{train}}$\\}
We study the behavior of the relative estimation error with respect to $\sigma$ and the number of training samples.
We compare OptEN with the {empirical estimator} $\widehat\xbs$, DP, NGCV, and BP. 
We use $m=500,\,d=100, \,h=10$ and \ref{data:D1}-\ref{data:D3} with $\sigma$ ranging from {$0.1$ to $0.5$} in the first experiment, whereas we vary the number of training samples $N$ from $20$ to $60$ in the second experiments as depicted in Figure \ref{fig:noise_levels}. 
Our method again outperforms other considered parameter selection rules.
On the other hand, the {empirical estimator} performs  slightly better than OptEN for larger noise levels (it is also better than the elastic nets solution with the optimal parameter), and it performs worse for lower noise levels.
This is essentially due to the fact that $\widehat{\xbs}$ is never truly sparse, but has a lower (thresholded) FDP.
Namely, as a projection onto an $h$-dimensional space the non-zero entries of $\widehat{\xbs}$ are very small, whereas the non-zero entries obtained with elastic net are larger and their size depending on the noise level.
The parameter $\alpha$ plays a similar role; for small $\alpha$ OptEN beats $\widehat\xbs$, and for larger $\alpha$ the situation is reversed.

\begin{figure}[ht!]
\centering
\begin{minipage}[b]{0.49\linewidth}
\centering
\includegraphics[width=\textwidth]{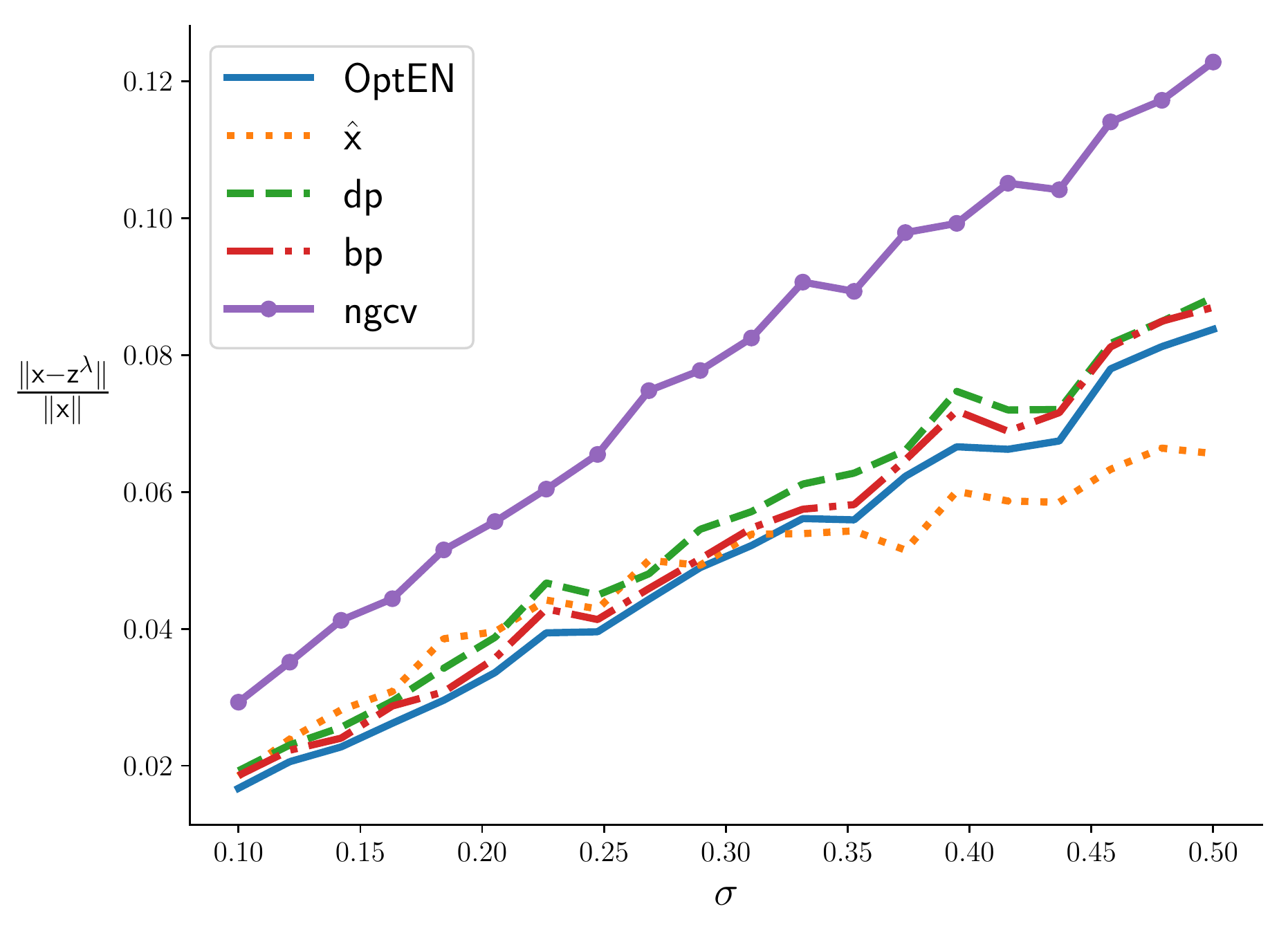}
\end{minipage}
\begin{minipage}[b]{0.49\linewidth}
\centering
\includegraphics[width=\textwidth]{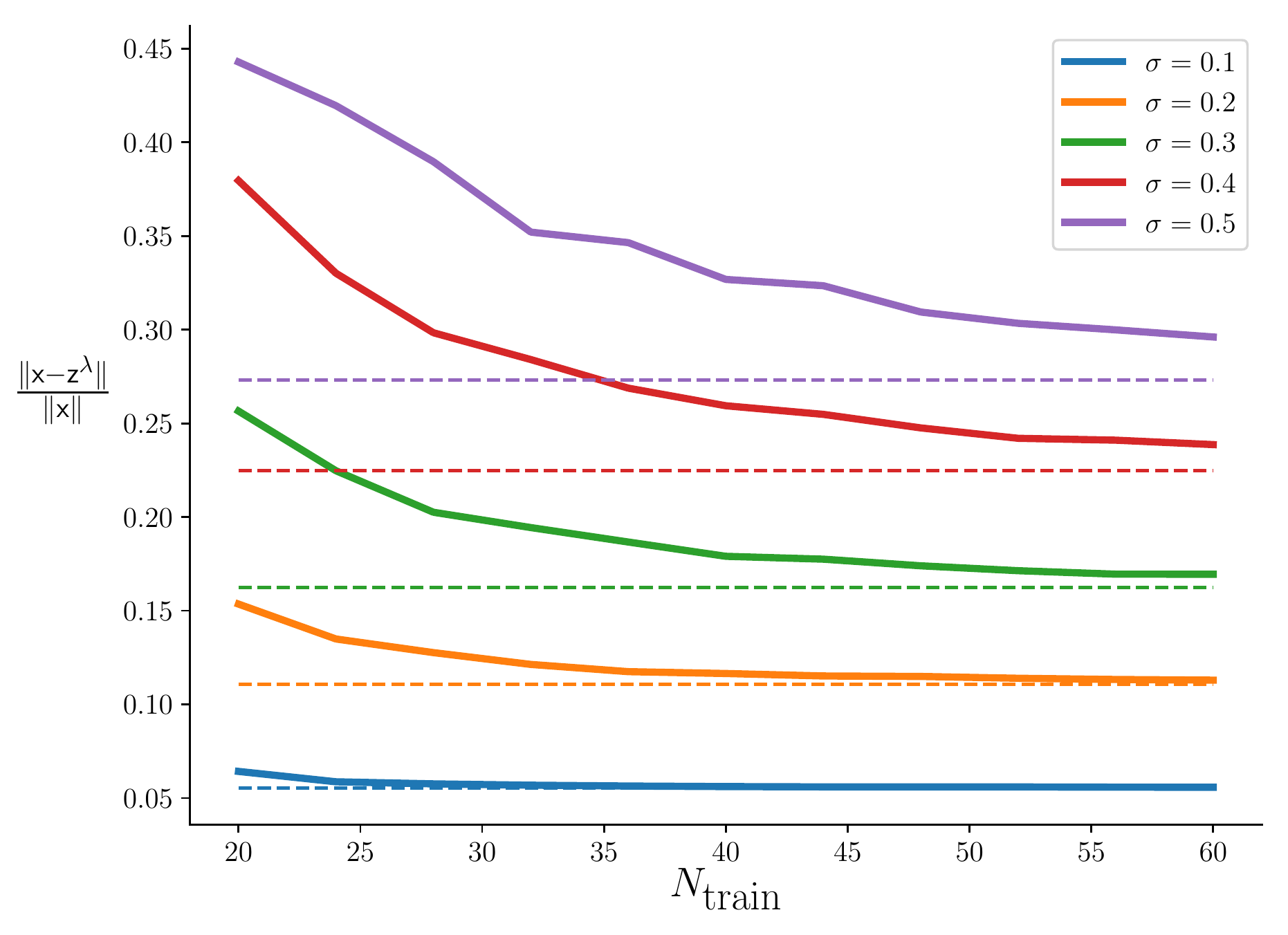}
\end{minipage}
\caption{In the left panel is the behavior of the {empirical estimator}, OptEN, discrepancy principle, balancing principle, and nonlinear GCV with respect to $\sigma$ is shown. 
In the right panel the behavior of our method for different values of $\sigma$ as the number of samples $N$ increases is shown. 
Dashed lines represent the error achieved by taking the true optimal parameter for the corresponding $\sigma$}
\label{fig:noise_levels}
\end{figure}

\paragraph{Non-injective matrices\\}
We now conduct experiments with non-injective matrices. 
The setting is as in Table \ref{tab:rando_label_3119}, where now $\Amat\in\bbR^{500\times 100}$ with $\rank\LRP{\Amat}=40$. 
As mentioned in Section \ref{subsec:qlf}, we test our method by minimizing the projected loss functional $\widehat{R}_\FP$ and the modified error functional $\widehat{R}_\FM$.
The results can be found in Table \ref{tab:rando_label_7614}.
Our method (using both the projected and modified functionals) again outperforms standard parameter selection rules in terms of the precision accuracy, and loses out to some methods when it comes to FDP and TPP.
We also observe that the performance of the {empirical estimator} deteriorates and $\widehat\xbs$ indeed should not be used as the solution itself but some additional regularization is required.
\begin{table}[th]
\centering
\ra{1.1} 
\begin{tabular}{cccccc}Method & $\frac{|t_{opt} - \widehat{t}|}{t_{opt}}$& $\frac{\|\xbs - \zbs^{\widehat{t}}\|}{\|\xbs\|}$& $FDP(\widehat{t})$& $TPP(\widehat{t})$& $\genfrac{}{}{0pt}{}{\textrm{computational}}{\textrm{time [s]}}$\\
\hline\hline
using $t_{opt}$ & 0 & 0.5704 & 0.4918 & 0.9450 & 0\\
\hline
{empirical estimator} $\widehat\xbs$ & N/A & 0.7978 & 0.8047 & \textbf{1.000} & N/A\\
projected OptEN & \textbf{0.0718} & \textbf{0.6033} & 0.507 & 0.930 & 8.16\\
modified OptEN& 0.0763 & 0.6046 & 0.497 & 0.926 & 15.57\\
DP & 0.1316 & 0.6343 & 0.528 & 0.926 & 1.93\\
ME & 0.3203 & 0.7234 & \textbf{0.290} & 0.765 & \textbf{0.15}\\
QO & 0.3167 & 0.7857 & 0.819 & 0.997 & 7.04\\
LC & 0.3389 & 0.7426 & 0.304 & 0.749 & 7.00\\
GCV & 0.3172 & 0.7865 & 0.819 & 0.997 & 14.04\\
NGCV & 0.3172 & 0.7865 & 0.819 & 0.997 & 7.06\\
BP & 0.2636 & 0.7179 & 0.757 & 0.974 & 35.01\\
ENBP & 0.2133 & 0.6549 & 0.362 & 0.847 & 3.69
\end{tabular} 
\caption{Comparison of errors for regularization parameter selection methods, with a matrix $\Amat\in\bbR^{500\times 100}$, $\rank(\Amat)=40$, and $h=10$, $\alpha=10^{-3}$, $\sigma=0.3$. The values are averages over $100$ independent runs}
\label{tab:rando_label_7614} 
\end{table}

\subsection{Image denoising}\label{sec:wavelet_denoising}
The task of image denoising is to find an estimate $ \boldsymbol Z$ of an unknown image $\boldsymbol{X}$ from a noisy measurement $\boldsymbol{Y}$, where 
$\boldsymbol{Y} = \boldsymbol X + \sigma\boldsymbol{\Xi}$, and $\boldsymbol{\Xi}$ denotes isotropic white noise. The goal is to improve the image quality by removing noise while preserving important image features such as edges and homogeneous regions.

There are a large number of methods addressing image denoisig , starting from 'classical' wavelet thresholding  \cite{DJ94, D95} and non-linear filters, to stochastic and variational methods \cite{CL97,MR1669536}. 
Since the primary goal of this paper is to evaluate how does the proposed approach perform as a parameter selection method for the elastic net, here we only compare  our method with other -art parameter selection methods for elastic nets, and do not compare elastic nets with image denoising methods in general. 
In particular, we compare OptEN with the discrepancy principle and the balancing principle (\emph{i.e.} the top performers from previous experiments).
In all cases the results show that OptEN has superior performance and selects nearly optimal parameters, see Table \ref{tab:iqas}. 

\paragraph{Wavelet-based denoising.}  
We denoise the noisy image $\boldsymbol{Y}$ by minimizing
\begin{equation}\label{eq:wavelet}  (1-t)\N{\boldsymbol{Z} - \boldsymbol{Y}}_2^2 + t( \N{\bbW\boldsymbol{Z}}_1 + \alpha\N{\boldsymbol{Z}}^2_2), \textrm{ for } \boldsymbol{Z}\in[0,1]^p,\end{equation}
 where $\bbW$ is the  wavelet transform using the family of \texttt{db4} wavelets, and $\alpha  = 10^{-3}$. %
Wavelet transform sparsify natural images, and we thus select the {empirical estimator} in the wavelet domain.
Moreover, in the limit with respect to the number of samples $N\rightarrow\infty$, the empirical projection $\widehat \Pi Y$ is for a given $h$  equivalent to a hard thresholding of $\bbW\boldsymbol{Y}$ that preserves its $h$ largest wavelet coefficients. 
Thus, for image denoising we do not use samples $\boldsymbol{Y}_i$ but instead only threshold $\bbW\boldsymbol{Y}$ for a well chosen $h$.
Here $h$ cannot be chosen by searching for a gap in $\bbW\boldsymbol{Y}$, since it most often does not exist.
Instead, we say that the \emph{true} $h$ is the one that minimizes the MSE of the reconstructed image.
In our first set of experiments the {empirical estimator} $\widehat{\boldsymbol{X}}$  is chosen by hard thresholding $\boldsymbol{Y}$, where $h$ is optimal.

\subsubsection{Denoising with an oracle \it{\normalfont{h}}}\label{sec:OracleH}

\paragraph{Data and learning setup.} 
We consider five grayscale images: \texttt{space shuttle}, \texttt{cherries}, \texttt{cat}, \texttt{mud flow}, and \texttt{IHC}, each of size $512\times512$ pixels. 
For BP and DP the regularization parameter is selected from a sequence of parameter values  $t_n = \frac{1}{1+\mu_0 q^n}$ with $\mu_0=1$, $q=0.95$, and $N_{\max} = 100,$ same as before. 
Moreover, we fix $\tau=1$, $\kappa=1/4$ for BP and DP, and provide them with the true noise level. 
For OptEN
the {empirical estimator} is computed with an oracle $h$, \emph{i.e.} the one returning the lowest MSE.

\paragraph{Comparison/Error.} 
We use two  performance metrics: peak signal-to-noise ratio (PSNR) and the similarity index (SSIM) between the original image $\boldsymbol{X}$ and the recovered version $\boldsymbol{Z}$.
PSNR is a standard pixel-based performance metric, defined through the MSE by
\[\quad \text{PSNR}(\boldsymbol{X},\boldsymbol{Z}) = 10 \log_{10}\LRP{\frac{\max_{1\leq i\leq p} X_i - \min_{1\leq i\leq p} X_i}{\text{MSE}(\boldsymbol{X},\boldsymbol{Z})}},  \text{MSE}(\boldsymbol{X},{\boldsymbol{Z}}) = \frac{1}{p} \N{\boldsymbol{X} - \boldsymbol{Z}}^2.\]
MSE and PSNR are ubiquitous in image and signal analysis due to their simplicity and suitability for optimization tasks, but are also infamous for their inability to capture features salient for human perception of image quality and fidelity \cite{WB+04}.
SSIM on the other hand, is a structure-based performance metric that tries to address this issue by using easy-to-compute structural statistics to estimate image similarity.
It is defined through
  \[\text{SSIM}(\boldsymbol{X},\boldsymbol{Z})=\LRP{\frac{2{\overline{\boldsymbol{X}}}\,{\overline{\boldsymbol{Z}}} + C_1}{{\overline{\boldsymbol{X}}}^2{\overline{\boldsymbol{Z}}}^2 + C_1}}\LRP{\frac{2 std(\boldsymbol{X}) std(\boldsymbol{Z}) + C_2}{std(\boldsymbol{X})^2std(\boldsymbol{Z})^2 + C_2}},\]
where $\overline{\boldsymbol{X}},\,\overline{\boldsymbol{Z}}$ are the means, and $std(\boldsymbol{X}),\,std(\boldsymbol{Z})$ are the standard deviations of pixels of corresponding images $\boldsymbol{X}$ and $\boldsymbol{Z}$, and $C_1,C_2$ are positive constants\footnote{We take $C_1 = 0.01,\, C_2 = 0.03$ by the convention of \texttt{python}'s \texttt{skimage} package}.

\begin{table*}\centering
\ra{1.1}
{\footnotesize
\begin{tabular}{@{}cccccccccccc@{}}
& \multicolumn{5}{c}{{\normalsize PSNR}} && \multicolumn{5}{c}{{\normalsize SSIM}} \\
\cmidrule{2-6} \cmidrule{8-12}  & $\genfrac{}{}{0pt}{}{\textrm{using}}{t_{\textrm{opt}}}$ & noisy 
& OptEN & DP  & BP && $\genfrac{}{}{0pt}{}{\textrm{using}}{t_{\textrm{opt}}}$ & noisy& OptEN & DP & BP  \\
\midrule
$\genfrac{}{}{0pt}{}{\textrm{\textsf{space}}}{\textrm{\textsf{shuttle}}}$\\
$\sigma=0.05$ & 32.33& 26.27& \textbf{32.32}& 30.26& 30.63 && 0.929& 0.687& \textbf{0.929}& 0.850& 0.862\\
$\sigma=0.075$ & 30.26& 22.87& \textbf{30.25}& 27.32& 27.69 && 0.908& 0.516& \textbf{0.908}& 0.748& 0.765\\
$\sigma=0.1$ & 28.77& 20.50& \textbf{28.76}& 25.32& 25.60 && 0.892& 0.394& \textbf{0.891}& 0.660& 0.676\\
\midrule
\textsf{cherries}\\
$\sigma=0.05$ & 35.80& 26.08& \textbf{35.79}& 30.52& 31.08 && 0.973& 0.647& \textbf{0.972}& 0.837& 0.855\\
$\sigma=0.075$ & 33.59& 22.63& \textbf{33.59}& 27.40& 27.77 && 0.964& 0.463& \textbf{0.964}& 0.720& 0.737\\
$\sigma=0.1$ & 31.94& 20.24& \textbf{31.82}& 25.26& 25.54 && 0.958& 0.339& \textbf{0.942}& 0.617& 0.633\\
\midrule
\textsf{cat}\\
$\sigma=0.05$ & 29.34& 26.02& \textbf{29.33}& 28.72& 28.93 && 0.890& 0.767& \textbf{0.890}& 0.861& 0.868\\
$\sigma=0.075$ & 27.06& 22.51& \textbf{27.03}& 25.85& 26.04 && 0.825& 0.612& \textbf{0.823}& 0.758& 0.767\\
$\sigma=0.1$ & 25.69& 20.04& \textbf{25.24}& 23.88& 24.03 && 0.771& 0.486& \textbf{0.741}& 0.664& 0.671\\
\midrule
\textsf{mud flow}\\
$\sigma=0.05$ & 28.38& 26.02& \textbf{28.37}& 28.20& 28.30 && 0.870& 0.777& \textbf{0.869}& 0.856& 0.860\\
$\sigma=0.075$ & 26.07& 22.50& \textbf{26.06}& 25.46& 25.60 && 0.795& 0.629& \textbf{0.795}& 0.755& 0.762\\
$\sigma=0.1$ & 24.71& 20.01& \textbf{24.46}& 23.58& 23.70 && 0.735& 0.506& \textbf{0.717}& 0.662& 0.668\\
\midrule
\textsf{IHC}\\
$\sigma=0.05$ & 29.33& 26.02& \textbf{29.33}& 28.73& 28.92 && 0.890& 0.767& \textbf{0.889}& 0.861& 0.868\\
$\sigma=0.075$ & 27.06& 22.51& \textbf{27.04}& 25.86& 26.04 && 0.825& 0.612& \textbf{0.823}& 0.759& 0.767\\
$\sigma=0.1$ & 25.70& 20.04& \textbf{25.39}& 23.85& 24.04 && 0.772& 0.486& \textbf{0.748}& 0.662& 0.671\\
\bottomrule
\end{tabular}}
\caption{Results on wavelet denoising of noisy images with different noise levels using elastic nets minimization. Each column defines the method used to select the regularization parameter. In bold is the method that achieved the best result. Columns titled \emph{noisy} correspond to PSNR and SSIM values of the initial noisy image. Columns titled \emph{using $\lambda_{\textrm{opt}}$} correspond to the best values achievable for the selected elastic nets functional, where we find the optimal parameter by a grid search on the true loss functional}
\label{tab:iqas}
\end{table*}

\begin{figure}
\includegraphics[width=\textwidth]{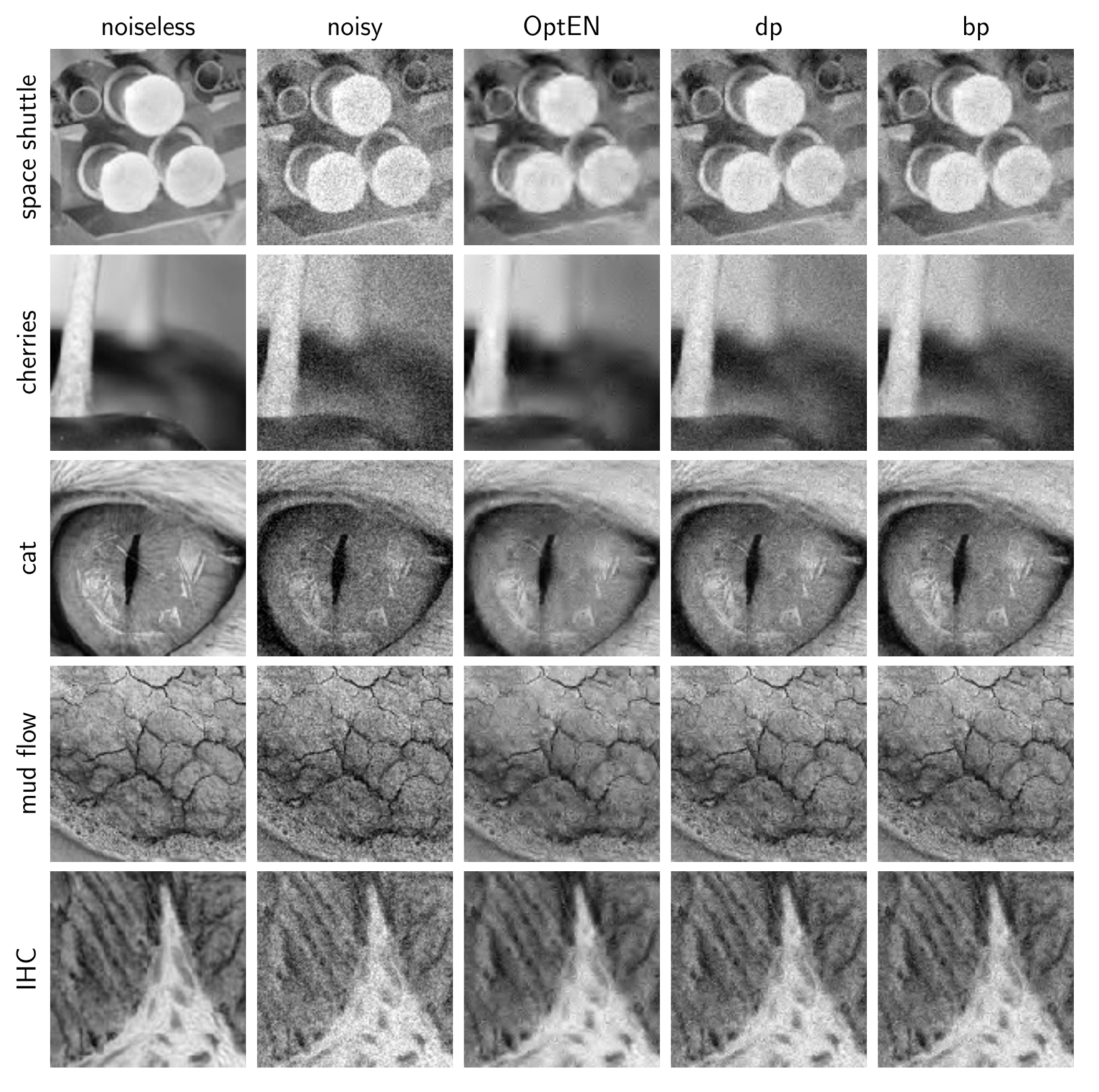}
\caption{Comparing the effects of denoising for different parameter selection rules}\label{fig:visual}
\end{figure}
Table \ref{tab:iqas} provides the PSNR and SSIM values generated by all algorithms on the considered images, while Figure \ref{fig:visual} shows the result of denoising on a $128\times128$ detail of each image for $\sigma=0.075$.
We can see that our method achieves the highest PSNR on all images and that this effect is more pronounced for larger noise values.
\subsubsection{Denoising with a heuristically chosen $h$}
{\begin{figure}
\centering
\includegraphics[width=0.5\textwidth]{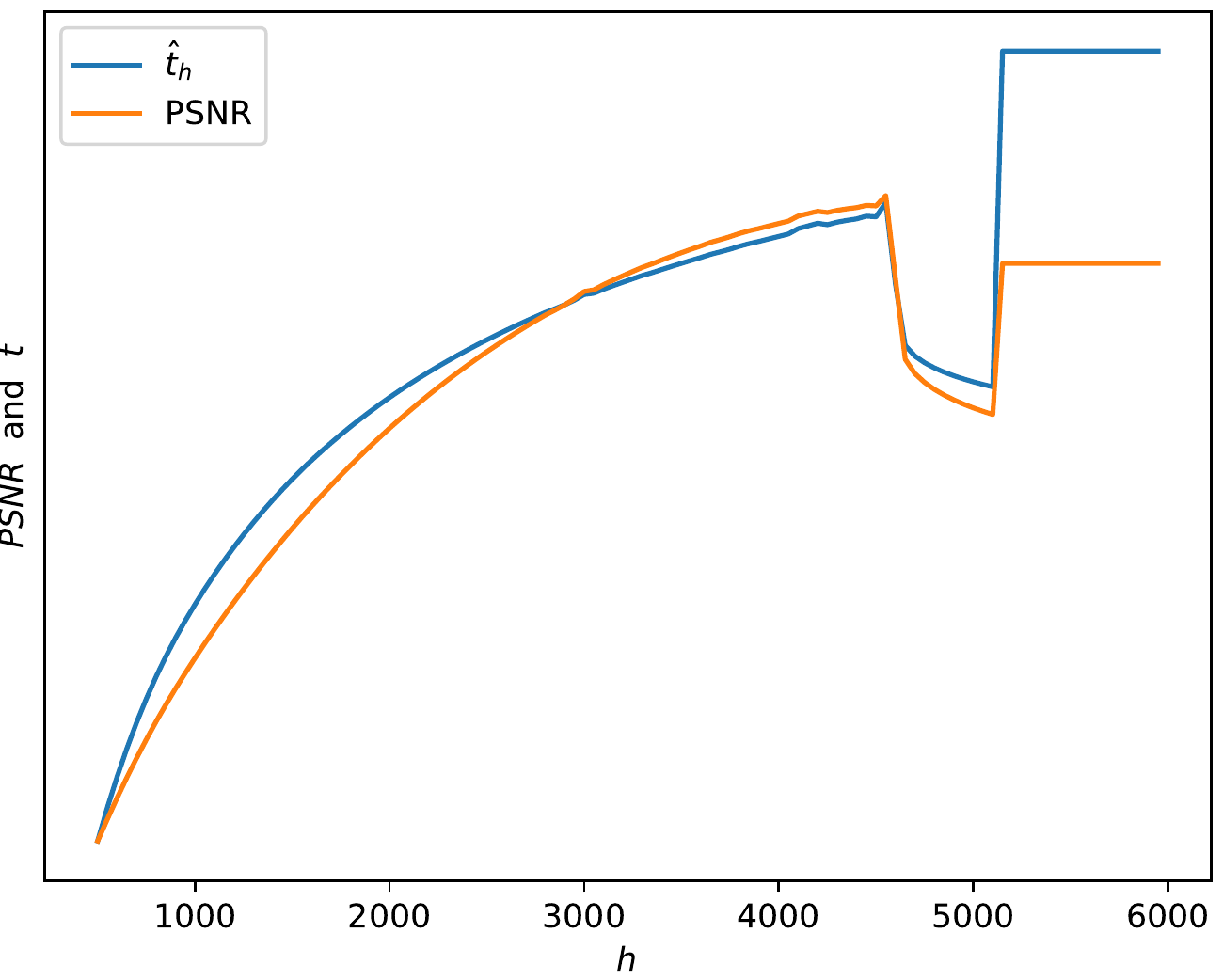}
\caption{A visual representation of our heuristic criteria for selecting $h$. Here we choose $h=4400$. The flat line at the end corresponds to $\hat{t}_k=1$.}
\label{fig:h_criteria}
\end{figure}
In this set of experiments we study the performance of our method in a situation where the optimal $h$ is not known \emph{a priori}.
We run experiments on a real-world dataset of brain images\footnote{Obtained from \url{http://nist.mni.mcgill.ca/?page_id=672}, therein referred to as \emph{group 2}}, 
 which  consists of \emph{in vivo} MRIs of 13 patients with brain tumor, taken pre-surgery.
For each patient we took an MRI slice, isolated the area around the brain and then added additional isotropic white noise with $\sigma\in\{0.05, 0.075, 0.1\}$.
{We then select $h$ by a heuristicaly driven procedure.
Namely, for each image $\boldsymbol{Y}$ we set an initial $h_0\in\bbN$ and determine $\hat{t}_0$ by performing Algorithm \ref{alg:algo}, where $\widehat{\boldsymbol{X}}_{h_0}$ is constructed by taking $h_0$ largest coefficients of $\bbW \boldsymbol{Y}$.
We then set $h_1= h_0+h_\text{step}$, repeat the procedure, and continue iteratively for $h_k$.
The iterations are stoppped once the corresponding  $\hat{t}_k$ start to decrease or become discontinuous (since heuristically this corresponds to a decrease in the PSNR of the corresponding elastic-net regularized solution)}.
$h_0$ and $h_\text{step}$ are chosen according to the size of the image.
The behaviour of this criteria can be seen in Figure \ref{fig:h_criteria}, and it shows that if $h$ is too large the empirical estimator is virtually the same as $\boldsymbol{Y}$. In other words, the minimizer of $\N{\boldsymbol{Z}^t-\widehat{\boldsymbol{X}}_h}$ is $t=1$ (\emph{i.e.} $\lambda=0$), which we observe in Figure \ref{fig:h_criteria}.}

The resulting reconstruction for $\sigma=0.1$ can be seen in Figure \ref{fig:brains} on four images.
The effects of denoising are visually not as striking as the results in Section \ref{sec:OracleH}.
We attribute this to the fact that PSNR gains with the best possible choice of parameter using elastic net are quite small, namely, PSNR of the noisy image improves only by around $5-7\%$, when taking the optimal parameter (see Table \ref{tab:brainpsnr}).
Other parameter selection rules (discrepancy principle and balancing principle) did not  improve the PSNR and are thus not presented.
\begin{table}[ht!]
\centering
\ra{1.1} 
\begin{tabular}{ccc}
using $t_{\textrm{opt}}$ & noisy & OptEN
\\
\hline\hline
28.608 & 27.297 & 28.432 \\
28.329 & 26.792 & 28.079 \\
28.221 & 26.735 & 27.935 \\
28.501 & 26.906 & 28.240
\end{tabular} 
\caption{PSNR values for the results in Figure \ref{fig:brains}}
\label{tab:brainpsnr} 
\end{table}

\section{Conclusion and future work}\label{sec:conclusion}
{In this paper, we presented an approach for the estimation of the optimal regularization parameter for elastic net.
The theoretical guarantees are possible only in simplified scenarios but we used insights gained therein to steer and create an efficient algorithm.
The algorithm exhibits excellent prediction accuracy, including in cases when there are no theoretical guarantees.
Comparison with state-of-the-art methods show a clear superiority of our method, under the studied testing scenarios.
Moreover, whereas other studied methods require adjsting a number of additional parameters in order to achieve satisfactory results, our method is entirely autonomous given a sufficient number of training samples.

\begin{figure*}[th!]
\centering
{%
  \includegraphics[clip,width=0.9\textwidth]{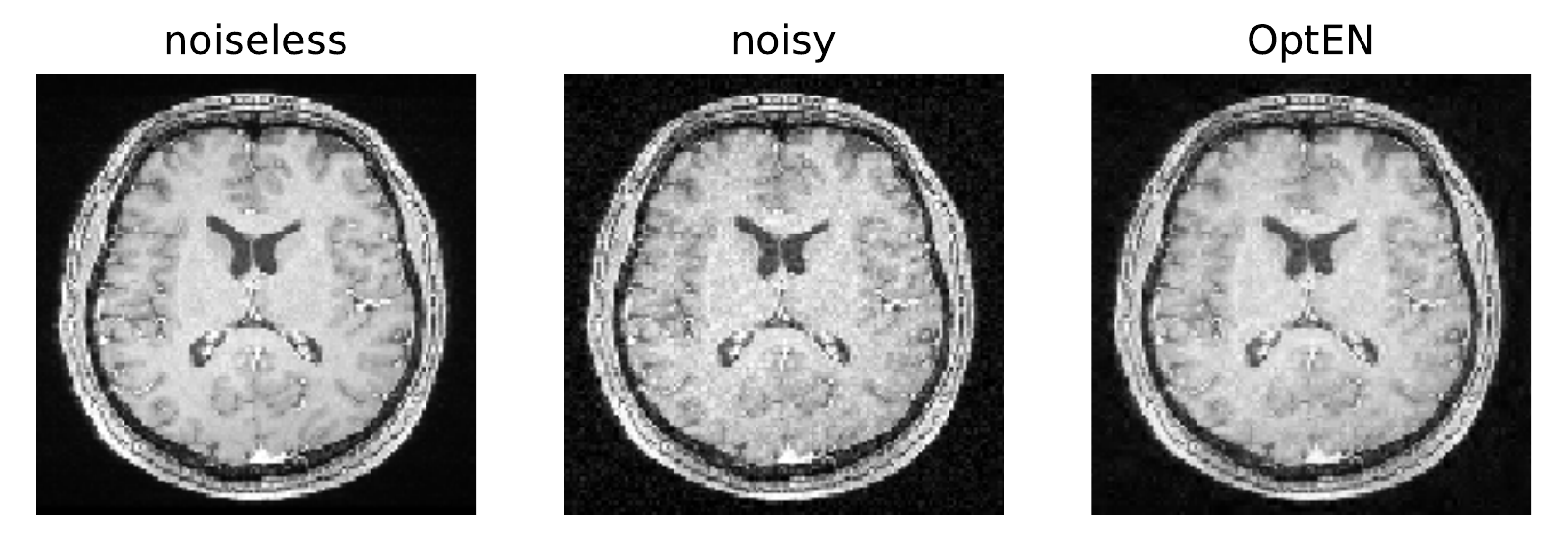}%
}
{%
  \includegraphics[clip,width=0.9\textwidth]{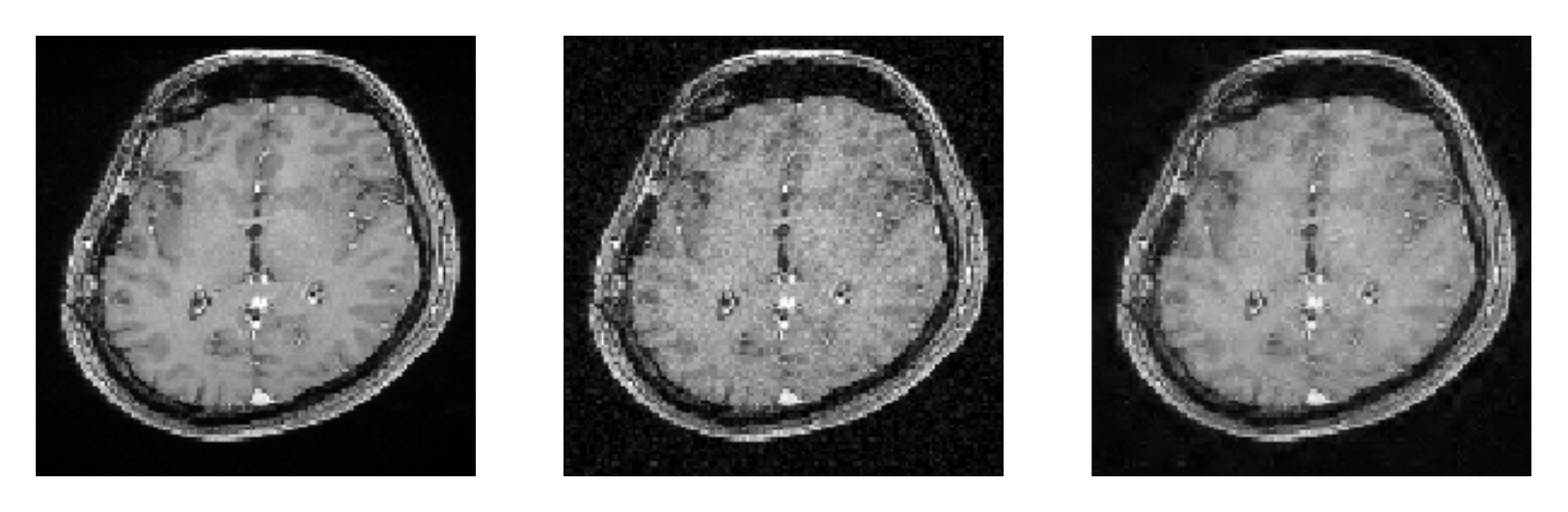}%
}
{%
  \includegraphics[clip,width=0.9\textwidth]{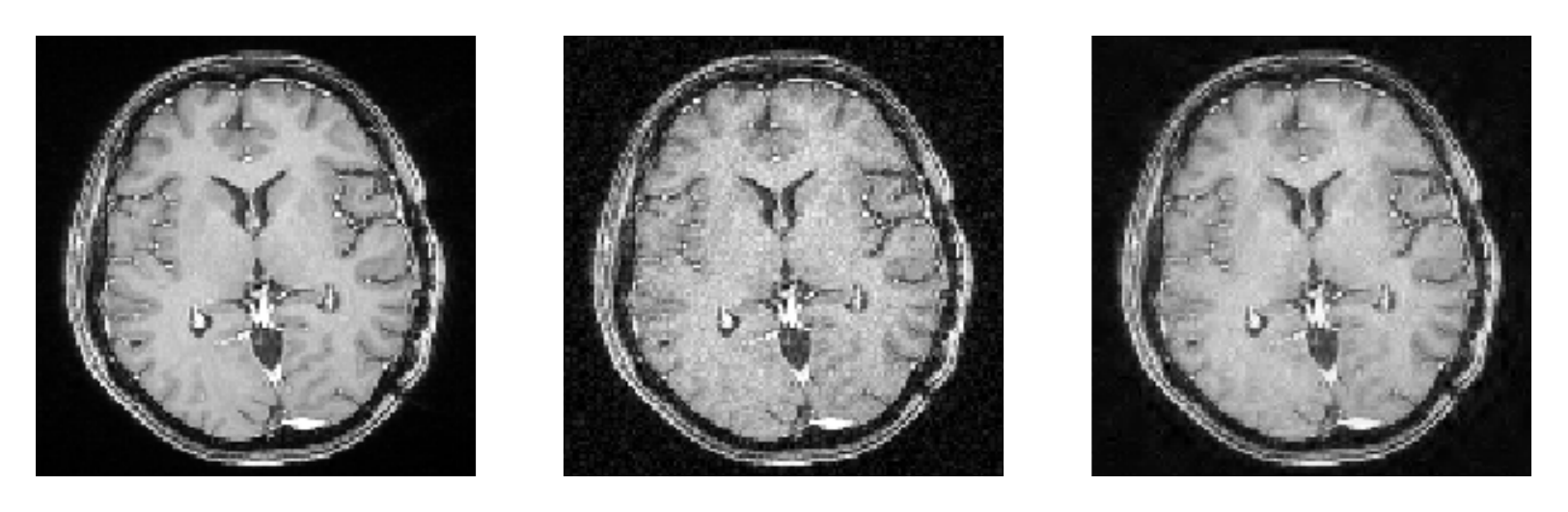}%
}
{%
  \includegraphics[clip,width=0.9\textwidth]{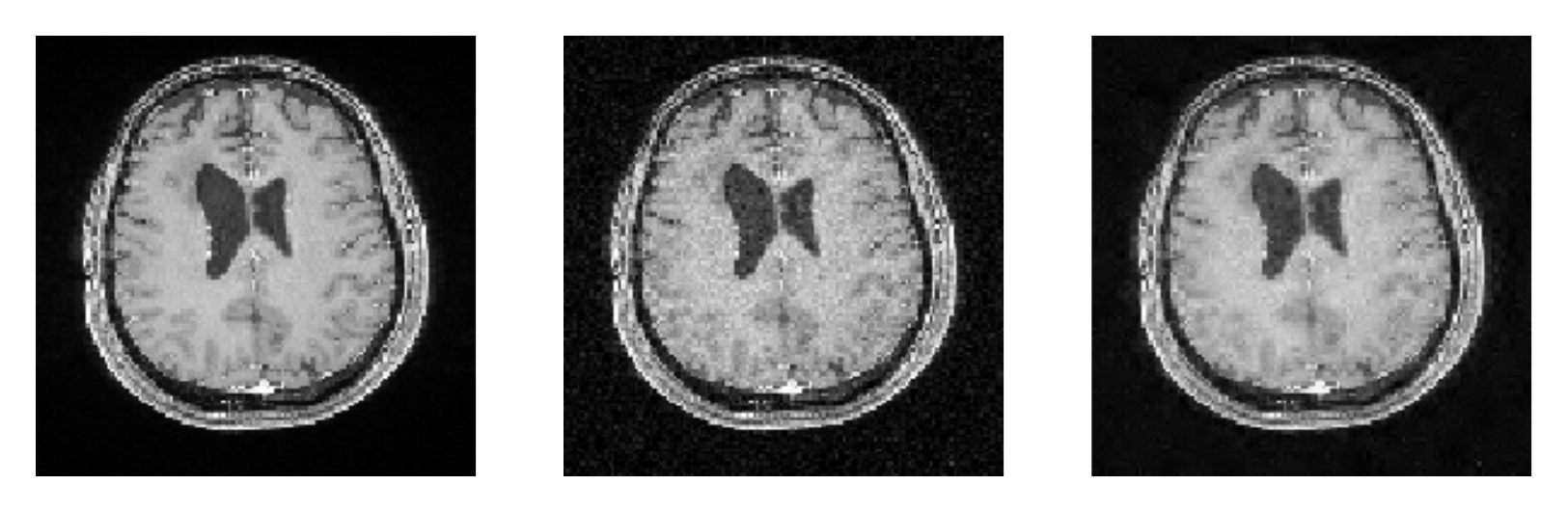}%
}
\caption{Denoising results on a brain image data set. The PSNR values are in Table \ref{tab:brainpsnr}}\label{fig:brains}
\end{figure*}

We aim to use the ideas presented in this paper in further studies.
Namely, we will study the behavior of the solution with respect to the other hyperparameter; $\alpha$, and consider other optimization schemes, predominantly focusing on imaging applications. 
We will also work on developing an optimization scheme for a joint minimization of both the regularisation functional and the loss functional for the regularization parameter.
}
\section*{Acknowledgements}
V. Naumova and Z. Kereta acknowledge  the support from RCN-funded FunDaHD project No 251149/O70.
E. De Vito is a member of the Gruppo Nazionale per l’Analisi Matematica, la Probabilità e le loro Applicazioni (GNAMPA) of the Istituto Nazionale di Alta Matematica (INdAM).
\appendix
\section{Appendix - Supplementary proofs}

\subsection{Proofs for Theorem \ref{lem:proj_error}}\label{sec:bounded_y_covs}
We will here add an analogue of equation \eqref{eqn:N_requirement} for the case of bounded $\ybs$.
Assume $\N{\ybs}_2\leq \sqrt{L}$ holds almost surely and consider a random matrix $\Rmat = \ybs^\top\ybs$.
Then $\bbE\Rmat=\Sigma(\ybs)$, and $\N{\Rmat} \leq L$.
Furthermore, $\Rmat^\top=\Rmat$ and 
\[ m_2(\Rmat) = \max \left\{ \N{\bbE[\Rmat\Rmat^\top]}, \N{\bbE[\Rmat^\top\Rmat]} \right\} = \N{\bbE\Rmat^\top\Rmat}\leq L\N{\Sigma(\ybs)}.\]
Let now $\ybs_i\sim\ybs$ and define a family of independent $m\times m$ matrices 
\[ \Rmat_i = \ybs_i^\top\ybs_i, \quad i=1,\ldots, N,\]
so that $\Rmat_i\sim \Rmat$.
The empirical covariance $\widehat{\Sigma}(\ybs)=\frac{1}{N}\sum_{i=1}^N\Rmat_i$ is then the matrix sampling estimator and by Corollary 6.2.1 from \cite{Tro15} we have that for all $s\geq 0$
\[\bbP\LRP{\N{\widehat{\Sigma}(\ybs)-\Sigma(\ybs) } \geq s } \leq 2m\exp\LRP{-\frac{Ns^2/2}{m_2(\Rmat) + 2Ls/3}}.\]
Writing now 
\( 2m\exp\LRP{-\frac{Ns^2/2}{m_2(\Rmat) + 2Ls/3}} = \exp(-u),\)
we have a quadratic equation for $s$, whose solution is
\[ s = \frac{4L\varepsilon + \sqrt{(4L\varepsilon)^2+ 72 Nm_2(\Rmat)\varepsilon }}{6N}\]
for $\varepsilon := u + 4\log(2m)$.
It then follows 
\begin{align*} 
s&\leq \frac{4L\varepsilon + 3\sqrt{2Nm_2(\Rmat)\varepsilon}}{3N} \leq \frac{\max(4L, 3\sqrt{2m_2(\Rmat)})}{3}\,\frac{\varepsilon + \sqrt{N\varepsilon}}{N}= C \LRP{\frac{u+\log(2m)}{N} + \sqrt{\frac{u+\log(2m)}{N}}}
\end{align*}
for $C=\frac{\max(4L, 3\sqrt{2m_2(\Rmat)})}{3}$.
Plugging it all together we have that with probability at least $1-\exp(-u)$
\[ \N{\widehat{\Sigma}(\ybs)-\Sigma(\ybs) } \lesssim \frac{u+\log(2m)}{N} + \sqrt{\frac{u+\log(2m)}{N}}.\]
Thus, provided $N\gtrsim u+\log(2m)$ we have \[\N{\widehat{\Sigma}(\ybs)-\Sigma(\ybs) }\leq \lambda_h/2.\]

\subsection{Computations for $\alpha\neq 1$ in Section \ref{sec:bernoulli_noise}}\label{app:alphaneq1}
Let \(\ybs = \xbs + \sigma\wbs,\) where \(\bbP\LRP{ \wsf_i = \pm 1} = \frac{1}{2},\)
and assume \( \xbs=(\xsf_1,\ldots,\xsf_h,0,\ldots,0)^\top,\) and \(\SN{\xsf_i}\geq 2\sigma,\) for  \(i=1,\ldots, h\).
In the following we will use $\vbs_{1:k}$ to denote a vector in $\bbR^k$ that consists of the first $k$ entries of a vector $\vbs\in\bbR^m$, and denote $\Veta=\sigma\wbs$.
In Section \ref{sec:error_analysis} we showed that the minimum of $R(t)=\N{\zbs^t(\ybs) - \xbs}_2^2$  and $\widehat R(t)=\N{\zbs^t(\ybs) - \widehat\xbs}_2^2$in each sub-interval $\CI_k$, for $k=1,\ldots,m$ of $[0,1]$ is of the form
\[ t^{\ast, k} = \frac{\sum_{i=1}^k a_id_i}{\sum_{i=1}^k a_ic_i}, \quad \widehat{t}^{\ast, k} = \frac{\sum_{i=1}^k a_i\widehat{d}_i}{\sum_{i=1}^k a_i\widehat{c}_i}\]
for 
\begin{align*}
  a_i &= \ssf_i(1+2\alpha\SN{\ysf_i}), \,\,\,
  c_i = \ssf_i+2\xsf_i(-1+\alpha) + 2\ysf_i,\,\,\, \widehat{c}_i = \ssf_i+2\widehat{\xsf}_i(-1+\alpha) + 2\ysf_i,\,
\,\,
  d_i = \ssf_i+2\alpha\xsf_i,\,\,\, \widehat{d}_i = \ssf_i+2\alpha\widehat{\xsf}_i,
\end{align*}  
  where $\ssf_i =\Sgn(\ysf_i)$.
  Denoting $\Verr =2(\widehat{\xbs}-\xbs)$, we write
\begin{align*} 
  \widehat{d}_i -d_i &= 2\alpha\LRP{\widehat{\xsf}_i - \xsf_i}=\alpha\mathsf{err}_i, \text{ and }
  \widehat{c}_i -c_i =2(\alpha-1)(\widehat{\xsf}_i -\xsf_i)=(\alpha-1)\mathsf{err}_i.
\end{align*}
We now have
\begin{align*}
t^{\ast,k} - \widehat{t}^{\ast,k} &= \frac{\sum_{i=1}^k a_id_i\sum_{i=1}^k a_i\widehat{c}_i - \sum_{i=1}^k a_i\widehat{d}_i \sum_{i=1}^k a_ic_i}{\sum_{i=1}^k a_ic_i\sum_{i=1}^k a_i\widehat{c}_i}\\
&=\frac{\sum_{i=1}^k a_i^2 \LRP{d_i\widehat{c}_i - \widehat{d}_ic_i} - \sum_{1\leq i<j\leq k} a_ia_j\LRP{d_i\widehat{c}_j+d_j\widehat{c}_i - \widehat{d}_ic_j-\widehat{d}_jc_i}}{\sum_{i=1}^k a_ic_i\sum_{i=1}^k a_i\widehat{c}_i}.
\end{align*}
Writing down each of the terms in the numerator we get
\begin{align*} 
  d_i\widehat{c}_i - \widehat{d}_ic_i=-2\alpha\mathsf{err}_i(\ysf_i-\xsf_i)  -\mathsf{err}_id_i = -\mathsf{err}_i a_i,
\end{align*}
where we use the fact that $c_i=d_i+2(\ysf_i-\xsf_i)$. 
We also get
\begin{align*}
  d_i\widehat{c}_j+d_j\widehat{c}_i - \widehat{d}_ic_j-\widehat{d}_jc_i &=\mathsf{err}_j\LRP{(\alpha-1)d_i-\alpha c_i} + \mathsf{err}_i\LRP{(\alpha-1)d_j -\alpha c_j}= -\LRP{\mathsf{err}_ja_i+\mathsf{err}_ia_j}.
\end{align*}  
Thus,
\begin{align*}
  \sum_{i=1}^k a_id_i\sum_{i=1}^k a_i\widehat{c}_i - \sum_{i=1}^k a_i\widehat{d}_i \sum_{i=1}^k a_ic_i &=-\LRP{\sum_{i=1}^k a_i^3\mathsf{err}_i + \sum_{1\leq i<j\leq k} a_ia_j (a_i\mathsf{err}_j+a_j\mathsf{err}_i)}\\
  &=-\sum_{i=1}^k \mathsf{err}_i\LRP{a_i^3 + a_i\sum_{i\neq j} a_j^2}=-\N{\abs_{1:k}}_2^2\sum_{i=1}^k a_i \mathsf{err}_i
\end{align*} 
Turning our attention to the denominator we have
\[\sum_{i=1}^k a_ic_i\sum_{i=1}^k a_i\widehat{c}_i = \sum_{i=1}^k a_i^2c_i\widehat{c}_i + \sum_{1\leq i<j\leq k} a_ia_j\LRP{c_i\widehat{c}_j+\widehat{c}_i c_j}=\LRP{\sum_{i=1}^k a_ic_i}^2 + (\alpha-1)\LRP{\sum_{i=1}^k a_ic_i} \sum_{j=1}^k a_j\mathsf{err}_j,\]
due to
\[ c_i\widehat{c}_j +\widehat{c_i}c_j = 2c_ic_j + (\alpha-1)\LRP{\mathsf{err}_jc_i+\mathsf{err}_i c_j}, \textrm{ and }\, c_i\widehat{c}_i = c_i^2 + (\alpha-1)c_i \mathsf{err}_i.\]
Putting it all together
and rewriting we have
\[ t^{\ast,k} - \widehat{t}^{\ast,k}
= \frac{-\N{\abs_{1:k}}_2^2\langle \abs_{1:k},{\Verr}_{1:k}\rangle}{ \LRP{\N{\abs_{1:k}}_2^2-2(\alpha-1)\langle \abs_{1:k},\Veta_{1:k}\rangle}\LRP{\N{\abs_{1:k}}_2^2- 2(\alpha-1)\langle\abs_{1:k}, (\ybs - \widehat{\xbs})_{1:k}\rangle} }   \]
 and recall $\Veta_{1:k}=(\ybs-\xbs)_{1:k}$.
Taking now $k=m$ we have by Cauchy-Schwartz inequality
\[ \SN{t_\text{opt} - \widehat{t}_\text{opt}} \leq \frac{\N{\abs}_2^3}{\SN{ \LRP{\N{\abs}_2^2-2(\alpha-1)\langle \abs,\Veta\rangle} \LRP{\N{\abs}_2^2- 2(\alpha-1)\langle\abs, (\ybs - \widehat{\xbs})\rangle} }}  \N{\xbs-{\widehat\xbs}}_2.\]
The term $\N{\xbs-\widehat\xbs}_2$ can be bounded as in Section \ref{sec:error_analysis}.
What is left is to bound the first factor. 
We compute
\[ {\SN{{\N{\abs}_2^2-2(\alpha-1)\langle \abs,\Veta\rangle} }}=\N{\abs}_2^2 \SN{1 -2(\alpha-1)\frac{\EUSP{\frac{\abs}{\N{\abs}_2}}{\Veta}}{\N{\abs}_2}}.\]
Provided\footnote{This holds for example if $(\alpha-1)^2\sigma^2m\lesssim  m + 4\alpha\N{\ybs}_1+4\alpha^2\N{\ybs}_2^2$} $\SN{2(\alpha-1){\EUSP{\frac{\abs}{\N{\abs}_2}}{\Veta}}} \leq \frac{\sqrt{2}}{2} {\N{\abs}_2}$ we have
\[\SN{1 -2(\alpha-1)\frac{\EUSP{\frac{\abs}{\N{\abs}_2}}{\Veta}}{\N{\abs}_2}}^{-1} \leq 2 \SN{1+2(\alpha-1)\frac{\EUSP{\frac{\abs}{\N{\abs}_2}}{\Veta}}{\N{\abs}_2}}\leq 2 \LRP{1+2\SN{\alpha-1}\frac{\N{\Veta}_2}{\N{\abs}_2}},\]
and
\[\SN{1 -2(\alpha-1)\frac{\EUSP{\frac{\abs}{\N{\abs}_2}}{\ybs-\widehat\xbs}}{\N{\abs}_2}}^{-1} \leq 2 \LRP{1+2\SN{\alpha-1}\frac{\N{\ybs-\widehat\xbs}_2}{\N{\abs}_2}}\leq2 \LRP{1+2\SN{\alpha-1}\frac{\N{\ybs}_2}{\N{\abs}_2}}. \]
where in the last line we used $\ybs-\widehat\xbs=(\Id-\widehat\Pi)\ybs$ and the fact $\N{\Id-\Pmat}_2=\N{\Pmat}_2$ for non-trivial (neither null nor identity) orthogonal projections $\Pmat$.
For $\alpha\geq1$ we have
\[ 2\SN{\alpha-1}\frac{\N{\ybs}_2}{\N{\abs}_2} \leq 1, \text{ and } 2\SN{\alpha-1}\frac{\N{\Veta}_2}{\N{\abs}_2} \leq 1,\]
using $\N{\abs}_2 \geq 2\SN{\alpha}\N{\ybs}_2$, and the signal-to-noise gap in the last inequality. 
On the other hand, for $0<\alpha<1$ we have $\N{\ybs}=\N{\xbs+\Veta}\lesssim \sqrt{h}+\sigma\sqrt{m}$, with high probability, giving
\[ \frac{\N{\abs}_2^3}{\SN{ \LRP{\N{\abs}_2^2-2(\alpha-1)\langle \abs,\Veta\rangle} \LRP{\N{\abs}_2^2- 2(\alpha-1)\langle\abs, \ybs - \widehat{\xbs}\rangle} }}\lesssim \frac{1}{\sqrt{m}}. \]
In conclusion, for $\alpha>0$ we have
\[ \SN{t_\text{opt} - \widehat{t}_\text{opt}} \lesssim \N{\Pi-\widehat{\Pi}}_2 + \sigma\sqrt{\frac{h}{m}},\]
as desired.
\bibliographystyle{plain}
\bibliography{biblioEN_ER.bib}

\end{document}